%% file: iclr2027_conference.tex
\documentclass{article} 
\usepackage{iclr2027_conference,times}
\usepackage{graphicx}
\usepackage{wrapfig}
\usepackage{multirow}
\usepackage{amsmath}
\usepackage{amsthm}
\usepackage{comment}
\usepackage{caption}
\usepackage{algorithm}
\usepackage{algorithmic}
\usepackage[dvipsnames]{xcolor}
\definecolor{peach}{RGB}{255,218,185}
\definecolor{lightblue}{RGB}{167, 199, 231}
\newtheorem{theorem}{Theorem}

\input{math_commands.tex}

\usepackage{hyperref}
\usepackage{url}

\title{Model-to-Data Distillation for Graph Neural Networks}

\author{Debolina Halder Lina
\\
Department of Computer Science, Rice University, Houston, TX 77005, USA
\AND
 Arlei Silva\\
Department of Computer Science \& Ken Kennedy Institute, Rice University, Houston, TX 77005, USA
}

\iclrfinalcopy 
\begin{document}

\maketitle

\begin{abstract}
Graph neural networks (GNNs) increasingly rely on sophisticated architectures and training procedures to achieve desirable properties such as high predictive performance, fairness, and robustness. However, these properties typically remain tied to the models that learn them, limiting their transferability to simpler models and downstream settings.
We introduce \textit{model-to-data (M2D) distillation}, a new distillation paradigm that transfers properties learned by a complex GNN teacher into graph data, enabling simpler models to recover them through standard training. M2D distillation explicitly trades model complexity for data complexity by jointly learning augmented node features and graph structure that encode the teacher's behavior. The resulting graph serves as a persistent medium for knowledge transfer and can be used with different downstream models.
We show that M2D distillation enables simple GNNs to approximate the behavior of substantially more sophisticated teachers, including fairness-aware GNNs, Graph Attention Networks, and Graph Transformers, while maintaining comparable predictive performance and transferring desirable properties of the teacher. 

\end{abstract}

\input{files/introduction}
\input{files/method}

\input{files/result}

\input{files/analysis}

\input{files/conclusion}

\subsection*{AI use statement}

In this work, we used generative AI tools for editing the paper. We have not used generative AI tools for assisting in the writing of proofs, provide feedback on research methodology and experiments, drafting parts of the paper, and the rest of the required disclosure tasks are not applicable to this work.
Additionally, we used generative AI tools for summarizing and analyzing existing literature. We have reviewed all AI-assisted work. We take responsibility for the final content of this work, including text, claims or artifacts produced with the aid of generative AI.

\bibliography{iclr2027_conference}
\bibliographystyle{iclr2027_conference}
\clearpage
\appendix
\input{files/appendix}

\end{document}

%% file: math_commands.tex
\usepackage{amsmath,amsfonts,bm}

\def\eqref#1{equation~\ref{#1}}

\def\1{\bm{1}}

\DeclareMathAlphabet{\mathsfit}{\encodingdefault}{\sfdefault}{m}{sl}
\SetMathAlphabet{\mathsfit}{bold}{\encodingdefault}{\sfdefault}{bx}{n}



%% file: files/introduction.tex
\section{Introduction}

Graph Neural Networks (GNNs) have achieved state-of-the-art performance across a wide range of graph learning tasks, including node classification, link prediction, and graph classification \citep{gilmer2017neural,kipf2016semi,velivckovic2018graph,hamilton2017inductive}. Alongside these advances, increasingly sophisticated models and training procedures have been developed to endow GNNs with desirable properties such as greater expressive power \citep{xupowerful}, fairness \citep{dai2021say}, robustness \citep{zugner2018adversarial}, calibration \citep{wang2021confident}, privacy \citep{sajadmanesh2023gap}, etc. Yet these properties typically remain tied to the specialized models and procedures used to obtain them. Such models may be inaccessible to researchers and developers due to intellectual property restrictions or impractical to deploy because of their computational resource requirements. This contrasts sharply with training data, which are persistent and model-agnostic: the same dataset can be reused to train generations of models with different architectures, objectives, and computational budgets. As a result, while data are often accessible and readily reusable, the desirable properties (e.g., expressive power, fairness) learned from them are not.

In this work, we ask whether properties learned by a GNN can themselves be transferred through data. Consider, for example, a Graph Convolutional Network (GCN) trained using a specialized fairness-aware procedure \citep{dai2021say}. Rather than distilling this model directly into a particular student, can we use the teacher to modify the training data such that a vanilla GCN \citep{kipf2016semi} trained on the same data is also fair? Our example is illustrated in Figure \ref{fig:m2dConcept}. More generally, given a teacher GNN model exhibiting a desired property, can we transform its training data so that new models inherit this property without requiring access to the teacher, its original training procedure, or an explicit property-specific objective?

We formulate this problem as \textit{model-to-data (M2D) distillation}. Unlike conventional knowledge distillation, which transfers knowledge from a teacher model directly to a particular student, M2D uses data as the medium through which learned properties are transferred. We position M2D within a unified model--data complexity space, illustrated in Figure~\ref{fig:modeldataspace}. Knowledge (K) distillation \citep{gou2021knowledge} reduces model complexity while approximately preserving learned behavior, corresponding to a vertical shift in this space, whereas Data (D) distillation \citep{lei2023comprehensive} reduces data complexity while preserving the ability to train effective models, corresponding to a horizontal shift. In contrast, M2D moves in the opposite direction: it trades model complexity for data complexity by transferring properties encoded in a sophisticated teacher into the training data, enabling simpler models to recover those properties through standard training.

\begin{figure}
    \centering\includegraphics[width=0.8\textwidth]{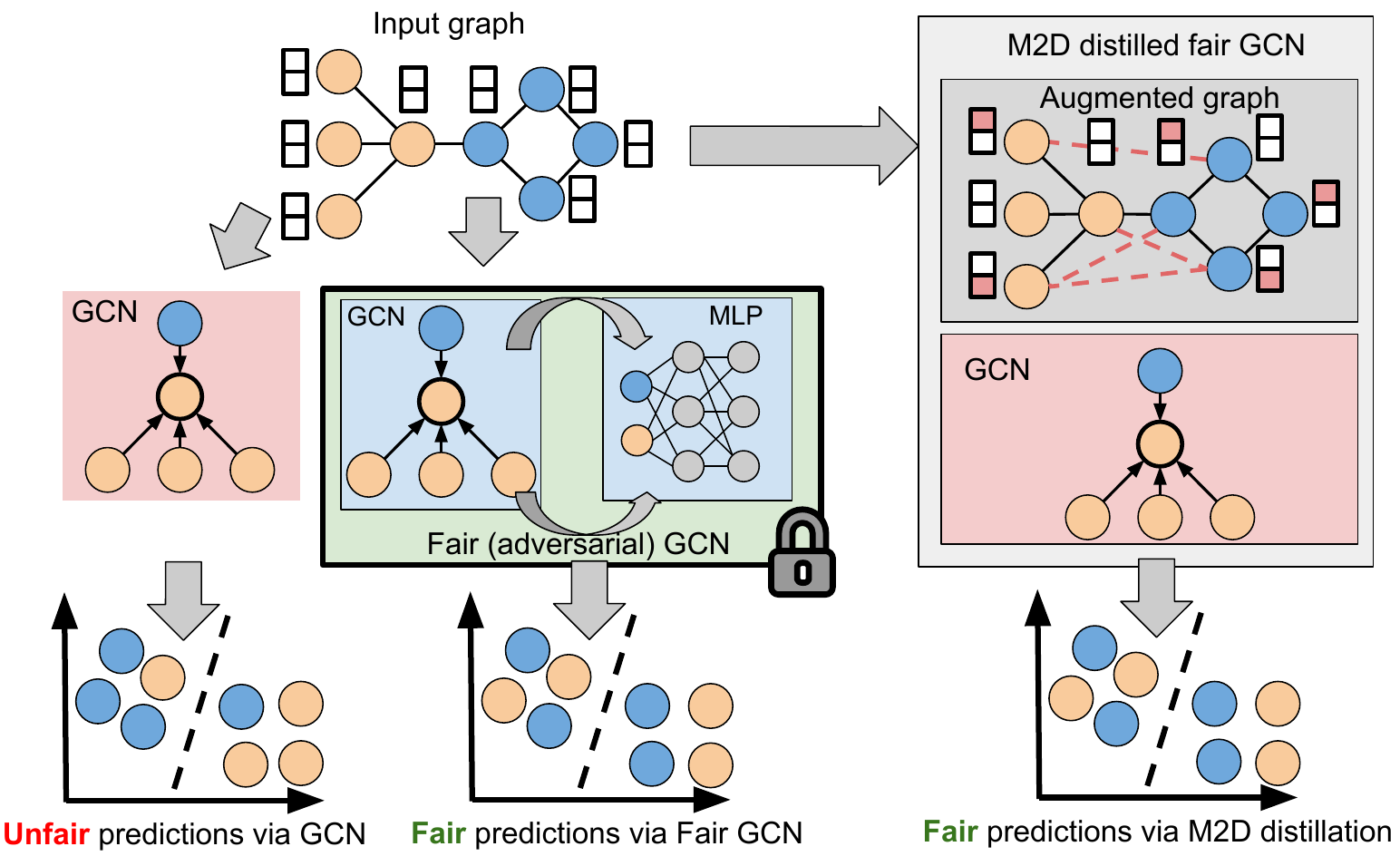}
    \caption{Model-to-data distillation (M2D) from a fair Graph Convolution Network (GCN) to a vanilla GCN. M2D distills the fair GCN into an augmented graph where features and structure are modified. Intuitively, the augmentation should remove the bias from the graph.
    }
    \label{fig:m2dConcept}
\end{figure}

\begin{wrapfigure}{r}{0.4\textwidth}
    \centering
    \vspace{-12pt}
    \includegraphics[width=\linewidth]{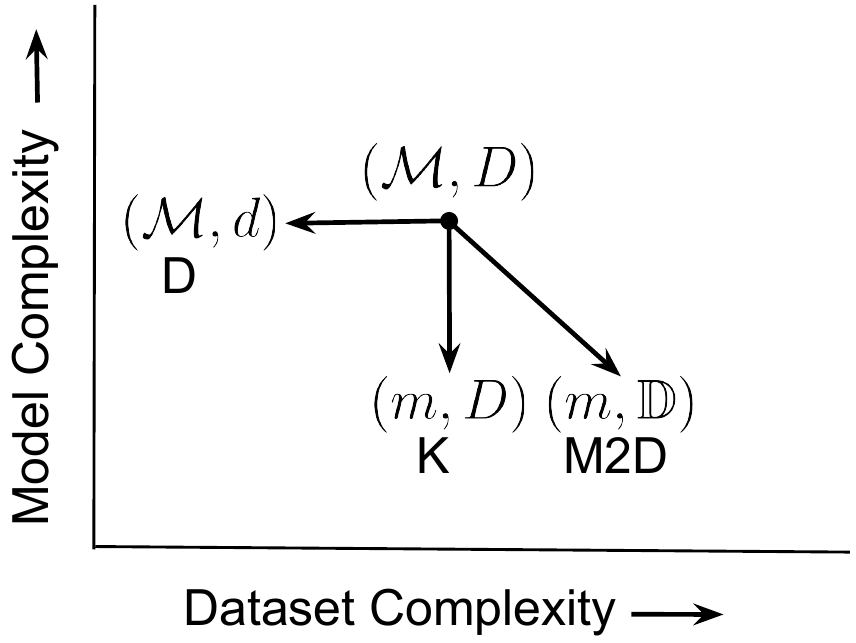}
    \vspace{-11pt}
    \caption{Distillation design space with data (x-axis) and model (y-axis) complexity. Knowledge (K) distillation reduces model complexity (vertical), and data (D) distillation reduces data complexity (horizontal). Model-to-Data (M2D) distillation transfers capacity from the model to data (diagonal).}
    \vspace{-15pt}
    \label{fig:modeldataspace}
\end{wrapfigure}




M2D is particularly useful in resource-constrained settings, where deploying the original teacher may be infeasible but training or deploying lightweight models is practical. Moreover, because the distilled property is mediated by the data, it can serve as a persistent form of behavioral knowledge that is reused across architectures and future generations of models. M2D can also provide a new mechanism for model interpretation and data auditing. By examining how a teacher modifies the original dataset to make a desired property emerge under standard training, we can identify features or structural patterns responsible for undesirable model behavior. 


We summarize our contributions as follows: (1) we introduce \textit{model-to-data (M2D) distillation}, a general framework for trading model complexity for data complexity by transferring desirable properties from sophisticated GNNs into graph data; (2) we instantiate M2D through an iterative framework that combines feature and structure learning with multi-objective optimization to encode teacher properties into the input graph; and (3) we empirically demonstrate that M2D enables simpler GNNs trained with standard procedures to approximate properties of more sophisticated teachers, including fairness-aware GNNs, Graph Attention Networks, and Graph Transformers.

\subsection{Related Work}

\textbf{Graph augmentation.}
Graph augmentation modifies graph structure or node features to improve GNN performance \citep{zhou2023opengsl,li2023gslb,zhu2021survey,attali2024rewiring}. While many approaches perturb graph topology by adding or removing nodes and edges~\citep{franceschi2019learning,chen2020iterative,fatemi2021slaps,zhao2022autogda,rongdropedge,gasteiger2019diffusion,dong2023towards}, some methods operate in feature space \citep{you2020graph,jeong2024igraphmix,guo2023interpolating}. For instance, FLAG learns adversarial perturbations of node features to improve robustness and generalization \citep{kong2022robust}, while graph contrastive learning methods commonly construct augmented views through feature masking or corruption \citep{you2020graph}. 

\textbf{Data distillation on graphs.}
Data distillation on graphs is commonly referred to as graph condensation (GC), which seeks to construct compact graphs that preserve the utility of the original data \citep{hashemi2024comprehensive}. GC methods can be broadly categorized according to their primary objectives. Most methods focus on performance preservation, learning condensed graphs such that GNNs trained on them achieve performance comparable to models trained on the original graph (e.g., GCond, SFGC, CTRL, OpenGC, GCSR, and KiDD) \citep{zhang2024two,gao2024graph,zheng2023structure,liu2024graph,jingraph,xu2023kernel}. Other methods emphasize generalization, constructing condensed graphs that transfer across architectures or tasks by preserving salient structural and feature information (e.g., GDEM and OpenGC) \citep{liugraph,gao2024graph}. A complementary line of work focuses on \textit{efficiency}, accelerating the condensation process through more efficient encoding, optimization, or synthetic-graph generation (e.g., DosCond, CaT, and EXGC) \citep{10.1145/3534678.3539429,10597970,10.1145/3589334.3645551}. Recent approaches have extended GC beyond predictive performance to properties such as fairness, with FGD and GCARe incorporating fairness constraints or regularization into the condensation process \citep{feng2023fair,Mao2023GCAReMS}, and robustness, with RobGC filtering noisy information while preserving core causal information for reliable downstream prediction \citep{11002712}.

\textbf{Knowledge distillation on graphs.}
Existing methods distill (1) output knowledge by aligning teacher and student predictions (e.g., TinyGNN, GFKD, DFAD-GNN, KDGA, CPF, and GLNN) \citep{yan2020tinygnn,deng2021graph,zhuang2022data,wu2022knowledge,yang2021extract,zhang2021graphless}; (2) structural knowledge by preserving local or global relationships encoded by the teacher (e.g., LSP, FreeKD, AlignAhead, GNN-SD, CKD, G-CRD, ROD, MSKD, Cold Brew, GKD, and NOSMOG) \citep{yang2020distilling,feng2022freekd,guo2022alignahead,chen2020self,wang2022collaborative,joshi2022representation,zhang2021rod,zhang2022multi,zheng2021cold,yang2022geometric,tian2023learning}; or (3) representation knowledge by aligning intermediate embeddings across models or layers (e.g., T2-GNN, and GraphAKD) \citep{huo2023t2,he2022compressing}. 
Some approaches combine distillation with augmentation by modifying structure, features, or graph views before or during knowledge transfer (e.g., KDGA, T2-GNN, NOSMOG, GAFD, and GraphGR) \citep{wang2025graph,lee2025cold,wu2022knowledge,huo2023t2,tian2023learning}.

M2D differs from conventional graph augmentation by distilling teacher behavior, rather than optimizing a predefined downstream objective. Moreover, our approach differs from knowledge and data distillation by transferring the teacher's behavior into an augmented graph rather than into a student model or compressed dataset. We note that M2D does not apply graph augmentation to simply facilitate teacher-to-student knowledge transfer, while the distilled knowledge ultimately resides in the student model. Instead, M2D reverses this role and treats the augmented graph itself as the target of distillation. Teacher behavior is encoded into its transformed features and structure, allowing the resulting graph to serve as a persistent, model-agnostic medium from which independently trained downstream models can recover the teacher's behavior without further access to the teacher.

%% file: files/method.tex
\section{Model to Data (M2D) Distillation}

\subsection{Problem Definition}
We propose Model-to-Data Distillation (M2D), a new distillation paradigm that transfers complexity from a Graph Neural Network (GNN) model to the data. 
M2D enriches the data so that a lightweight GNN can achieve performance comparable to a more complex one (see Figure \ref{fig:modeldataspace}). 

Let $\mathcal{G} = (\mathcal{V},\mathcal{E},\mathbf{X})$ be an undirected graph where $\mathcal{V}$ is the set of nodes, $\mathcal{E} \subseteq \mathcal{V} \times \mathcal{V}$ is the set of edges, and $\mathbf{X} \in \mathbb{R}^{n \times d}$ are the node attributes. The matrix $\mathbf{A} \in \mathbb{R}^{n \times n}$ is the adjacency matrix of $\mathcal{G}$ where $\mathbf{A}_{uv} = 1$ if there is an edge between $u$ and $v$ and $\mathbf{A}_{uv} = 0$, otherwise.
Let $f_T$ be the teacher model with parameter $\theta_T$ and $f_s$ be a student model with parameter $\theta_s$ where $|\theta_T| > |\theta_s|$. Moreover, let the prediction of the teacher model and the student model be $\hat{\mathbf{y}}_T$ and $\hat{\mathbf{y}}_s$, respectively. 

M2D learns a transformation $f_g$ with parameters $\theta_g$ that enriches the original graph data $\mathcal{G}$ by modifying the graph structure or the node features, yielding $\Tilde{\mathcal{G}} = ({\mathcal{V}},\tilde{\mathcal{E}},\tilde{\mathbf{X}})$. We learn the graph transformation $f_g$ and the student model $f_s$ jointly by optimizing the following bilevel objective:

\begin{equation}
\begin{aligned}
\min_{\theta_g}\quad 
\mathcal{L}_{\mathrm{dis}}\!\left(f_T(\mathcal{G}),f_s(\tilde{\mathcal{G}};\theta_s^{*})\right)
-
\mathcal{S}\!\left(\mathcal{G},\tilde{\mathcal{G}}\right)
\quad \text{s.t.}\quad
\theta_s^{*}
=
\arg\min_{\theta_s}
\mathcal{L}_{\mathrm{cls}}\!\left(f_s(\tilde{\mathcal{G}}),\mathbf{y}\right).
\end{aligned}
\label{eq:objective}
\end{equation}

where  $\tilde{\mathcal{G}} = f_g(\mathcal{G};\theta_g)$ is the augmented graph, $\mathcal{L}_{\mathrm{dis}}$ measures the difference between the outputs of the teacher $f_T$ and student $f_s$, $\mathcal{L}_{\mathrm{cls}}$ is the classification loss between true labels $\mathbf{y}$ and student output, and $\mathcal{S}$ is a function that computes the similarity between $\mathcal{G}$ and $\Tilde{\mathcal{G}}$. The term $\mathcal{S}(\mathcal{G}, \tilde{\mathcal{G}})$ enforces that the model knowledge is distilled into the data while introducing only minimal modifications to the original graph. In general, both the teacher model $f_T$ and the student model $f_s$ can be arbitrary architectures. The framework can also be generalized for a family of student models. In this work, however, we restrict our focus to GNNs and Graph Transformers (GTs) as teacher models and GNNs and MLPs as student models. Our framework assumes only partial access to teacher's logits. Let $\mathcal{V}_\mathrm{tr} \in \mathcal{V}$ denote the set of nodes for which both teacher logits and ground-truth labels are available.

\begin{figure}[htbp]
    \centering
    \includegraphics[width=0.8\textwidth]{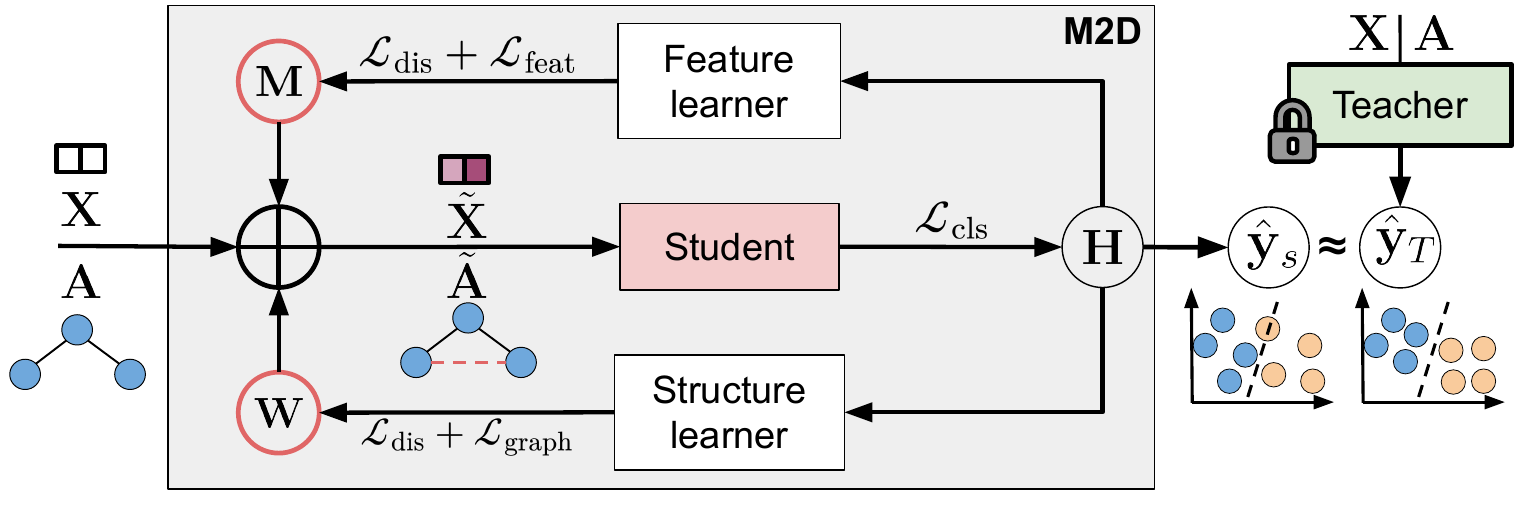}
    \caption{The feature and structure learners iteratively refine features and structure by optimizing $\mathcal{L}_{\mathrm{dis}},$ and $\mathcal{L}_{\mathrm{dis}}, \mathcal{L}_{\mathrm{graph}}$, respectively. The student model is then trained on the learned graph, optimizing $\mathcal{L}_{\mathrm{cls}}$ to align with the teacher while improving predictive performance.}
    \label{fig:flow}
\end{figure}

\subsection{Proposed Model}
\label{sec:m2d}
 We propose a solution to the M2D optimization problem of Eq \ref{eq:objective}. The learning process alternates between (i) updating the graph topology and node features using current embeddings and (ii) updating embeddings using the augmented graph. Let the student GNN $f_s$ generate node embedding $\mathbf{H}_s \in \mathbb{R}^{n \times d_h}$. We decompose the transformation $f_g$ into two components: a feature Learner $f_\phi$ that learns new node features and a structure learner $f_a$ that learns new adjacency weights. Figure \ref{fig:flow} shows the flow diagram of the proposed model.

\subsection{Learning node features and graph structure}
\paragraph{Feature Learner:}
Let $f_\phi:\mathbb{R}^{d_h}\rightarrow\mathbb{R}^{d}$ be a learnable feature masking function that generates a node-specific continuous feature mask from the student node embeddings. At iteration $t$, the mask is computed from node embeddings $\mathbf{H}_s$ at time $t-1$ as
\begin{equation*}
    \mathbf{M}^t = \sigma\!\left(f_\phi(\mathbf{H}^{t-1}_s)\right)
    \in [0,1]^{n\times d},
\end{equation*}
where $\sigma(\cdot)$ denotes the sigmoid function. The mask assigns an importance weight to each feature of every node. The augmented feature matrix is then obtained through element-wise multiplication,
\begin{equation*}
    \tilde{\mathbf{X}}^t
    =
    \mathbf{X}\odot\mathbf{M}^t
    \in
    \mathbb{R}^{n\times d},
\end{equation*}
where $\odot$ is the Hadamard product. Initial student node embeddings are obtained as
$\mathbf{H}_s^0=f_s(\mathbf{X},\mathbf{A})$.

\paragraph{Structure Learner:} We learn a new graph structure by predicting pairwise node relationships from their embeddings. Given node representations $\mathbf{H}^{t -1 }_s$, let $f_a:\mathbb{R}^{2d_h}\rightarrow[0,1]$ be a function that learns a new weighted adjacency matrix. The learned matrix is defined element-wise as:
\begin{align*}
\mathbf{W}^t_{uv} &= \sigma \big(f_a(\mathbf{h}^{t -1 }_{s,u}, \mathbf{h}^{t -1 }_{s,v})\big), 
\quad \mathbf{W}^t \in [0,1]^{n \times n},
\end{align*}
where $\mathbf{h}^{t -1 }_{s,u}$ and $\mathbf{h}^{t -1 }_{s,v}$ are embeddings of node $u$ and $v$ at time $t-1$.
For directed graphs, $f_a$ operates on the concatenation of node embeddings:
\begin{align*}
\mathbf{W}^t_{uv} &= \sigma \big(f_a([\mathbf{h}^{t -1 }_{s,u} \,\Vert\, \mathbf{h}^{t -1 }_{s,v}])\big).
\end{align*}

For undirected graphs, we enforce permutation invariance by symmetrizing the input:
\begin{align*}
\mathbf{W}^t_{uv} &= \sigma \big(f_a(\mathbf{h}^{t -1 }_{s,u} + \mathbf{h}^{t -1 }_{s,v} \,\Vert\, ~ 
\lvert \mathbf{h}^{t -1 }_{s,u} - \mathbf{h}^{t -1 }_{s,v} \rvert) \big),
\end{align*}
which ensures $\mathbf{W}^t_{uv} = \mathbf{W}^t_{vu}$. The final weighted adjacency matrix at iteration $t$ is a combination of the original graph with trained and untrained weights:

\begin{align*}
    \tilde{\mathbf{A}}^t_{uv} = (1 - \gamma) \mathbf{A}_{uv} + \gamma ((1 - \beta)\mathbf{W}^t_{uv} + \beta s(\mathbf{x}_u,\mathbf{x}_v)), 
\end{align*}
where $\gamma$ and $\beta$ are hyperparameters, and $s(\mathbf{x}_u,\mathbf{x}_v)$ is the cosine similarity between $u$ and $v$ based on original features. We then use $\tilde{\mathbf{X}}^t$ and $\tilde{\mathbf{A}}^t$ to generate representations $\mathbf{H}^t_s = f_s (\tilde{\mathbf{X}}^t,\tilde{\mathbf{A}}^t)$ at iteration $t$.
The final graph is generated either when $\mathcal{L}_{\mathrm{cls}}$ converges or after a fixed number of iterations.
\subsection{Training Objective}
\label{sec::training}

Let $\tilde{\mathcal{G}}^t\!=\!(\mathcal{V},{\tilde{\mathbf{A}}}^t,\tilde{\mathbf{X}}^t)$ be the augmented graph at time $t$. The student encoder $f_s$ produces node representations $\mathbf{H}^t_s\!=\!f_s({\tilde{\mathbf{A}}}^t,\tilde{\mathbf{X}}^t)$, which are mapped to class logits $\hat{\mathbf{y}}^t_s\!=\!f_{c}(\mathbf{H}^t_s)$ using a classifier $f_c$. We partition the model parameters into groups
$\Theta_s\!=\!\{\theta_s,\theta_c\}$ and $\Theta_g\!=\!\{\theta_\phi,\theta_a\}$, where $\Theta_s$ denotes the student encoder and classifier parameters and $\Theta_g$ denotes the graph augmentation parameters. 

\paragraph{Classification loss $\mathcal{L}_{\mathrm{cls}}$:} We use the standard supervised cross-entropy loss.
\begin{align}
\mathcal{L}_{\mathrm{cls}} 
= -\frac{1}{|\mathcal{V}_{\mathrm{tr}}|}
\sum_{u \in \mathcal{V}_{\mathrm{tr}}}
\mathbf{y}_u^\top \log \hat{\mathbf{y}}^t_{s,u}.
\end{align}

\paragraph{Distillation loss $\mathcal{L}_{\mathrm{dis}}$:}
We align student and teacher predictive distributions using temperature-scaled soft targets via the Kullback–Leibler (KL) Divergence. Let the teacher's logits be $\hat{\mathbf{y}}_T$, then:
\begin{align}
    \mathcal{L}_{\mathrm{dis}} = \tau^2 \frac{1}{|\mathcal{V}_{\mathrm{tr}}|}\sum_{u \in \mathcal{V}_{\mathrm{tr}}}D_\text{KL}\left(\mathrm{softmax}\left(\frac{\hat{\mathbf{y}}_{T,u}}{\tau}\right) , \mathrm{softmax}\left( \frac{\hat{\mathbf{y}}^t_{s,u}}{\tau}\right)\right).
\end{align}
where $\tau$ smooths the distributions, exposing inter-class similarity and teacher uncertainty not captured by teacher logits \citep{hinton2015distilling}. The $\tau^2$ term preserves gradient magnitudes.

\paragraph{Feature reconstruction loss $\mathcal{L}_{\mathrm{feat}}$:}
To ensure the learned mask does not drastically change the original features, we calculate the feature reconstruction loss $\mathcal{L}_{\mathrm{feat}}$ as follows:
\begin{equation}
    \mathcal{L}_{\mathrm{feat}} = || \mathbf{X} - \tilde{\mathbf{X}}^t || ^2_F
\end{equation}
where $||.||_F$ is the Frobenius norm.
\paragraph{Graph regularization loss $\mathcal{L}_{\mathrm{graph}}$:}
Following \cite{kalofolias2016learn}, we impose additional constraints on $\mathbf{W}^t$ through our graph structure loss $\mathcal{L}_{\mathrm{graph}}$:
\begin{align}
    \mathcal{L}_{\mathrm{graph}} = -\frac{\lambda_{\mathrm{deg}}}{n}\mathbf{1}^T\log(\mathbf{W}^t\mathbf{1}) + \frac{\lambda_{\mathrm{sparse}}}{n^2}|| \mathbf{W}^t||^2_F.
\end{align}
where the first term penalizes zero-degree nodes and the second term enforces sparsity on the learned adjacency, encouraging the augmented graph to remain close to the original topology. 

The feature reconstruction loss and structural regularization are optimized only with respect to the graph augmentation parameters $\Theta_g$. Therefore, their gradients with respect to the parameters $\Theta_s$ are constrained to vanish: $\nabla_{\Theta_s} \mathcal{L}_{\mathrm{feat}} = 0$ and $\nabla_{\Theta_s} \mathcal{L}_{\mathrm{graph}} = 0$. At iteration $t$, the student parameters $\Theta_s$ are optimized using the classification loss, whereas the graph modifier parameters $\Theta_g$ are optimized using the distillation loss and the structural regularization term:
\begin{align*}
    \Theta_s^* = \arg\min_{\Theta_s} ~\mathcal{L}_{\mathrm{cls}}, \quad
    \Theta_g^* = \arg\min_{\Theta_g} \mathcal{L}_{\mathrm{dis}} + \alpha \mathcal{L}_{\mathrm{feat}}+\mathcal{L}_{\mathrm{graph}}. 
\end{align*}
where $\alpha$ is a hyperparameter that regulates the impact of the feature reconstruction loss. Importantly, the graph modifier is not optimized using the classification loss. Its learning signal comes exclusively from the teacher–student distillation objective, ensuring that the augmented graph is explicitly guided by the teacher’s behavior. 

The M2D training algorithm and computational complexity are detailed in the Appendix.

\section{Theoretical Analysis}
 
Theorem~\ref{The:fairness} shows that the fairness gap between a downstream model and the teacher is bounded by the distillation error and by how closely the original student and downstream model, when trained with cross-entropy, approximate the optimal predictor on the transformed graph. Theorem \ref{the:feature_weight} shows that M2D transfers the teacher's attention-weighted neighborhood information into both learned feature masks and edge weights, up to residuals from local linearization and imperfect distillation.

\begin{theorem}
Let $\tilde{\mathcal{G}}$ be the graph learned by M2D, and let $\hat{Y}_T$ and $\hat{Y}_s$ denote the teacher and M2D student predictions. For binary sensitive attribute $a\in\{0,1\}$, define
$
\mathrm{DP}(f)
=
\left|
\mathbb{E}[\hat{Y}_f\mid a=0]
-
\mathbb{E}[\hat{Y}_f\mid a=1]
\right|.
$
Let $\kappa_g$ denote the teacher--student KL divergence and $\xi_g(f)$ the excess log-loss of model $f$, both conditioned on $a=g$. Then any independently trained downstream student $f_{s'}$ on $\tilde{\mathcal{G}}$ satisfies
$
\left|\mathrm{DP}(f_{s'}) - \mathrm{DP}(f_T) \right|
\le
\sum_{g\in\{0,1\}}
\left(
\sqrt{\frac{\kappa_g}{2}}
+
\sqrt{\frac{\xi_g(f_s)}{2}}
+
\sqrt{\frac{\xi_g(f_{s'})}{2}}
\right).
$
\label{The:fairness}
\end{theorem}

\begin{theorem}
Let the teacher be a GAT or GT with attention coefficients
$\alpha_{ij}$ and let the teacher and student have linear prediction
heads.  Assume a local first-order approximation of the teacher representation,
finite teacher logits, and a minimum-$\ell_2$-norm student solution. Then there exists a linear transformation $Q$ and residual
$\mathbf r_i$ such that
$
\mathbf m_i=\sigma(f_\phi(P_i)), \quad \text{and} \quad
w_{ij}=\sigma\left(f_a(P_i+P_j,|P_i-P_j|)\right),
$
where
$
P_i=\sum_{k\in\mathcal N(i)}\alpha_{ik}Q\mathbf x_k+\mathbf r_i$.
\label{the:feature_weight}
\end{theorem}

%% file: files/result.tex
\section{Experimental Evaluation}

\begin{figure}[htbp]
    \centering
    \includegraphics[width=0.9\textwidth]{files/images/f_acc_plot2.pdf}
    \caption{Accuracy–fairness trade-offs for FairGNN, FairVGNN, and NIFTY and their distilled variants. The y-axis shows accuracy, while the x-axis reports 1 - DP and 1 - EQOP. Points closer to the top-right indicate a better trade-off. Results show that bias in \textsc{NBA} is primarily feature-driven, whereas \textsc{Simulation} and \textsc{German} datasets exhibit bias in both features and topology.}
    \label{fig:fairness_accu}
\end{figure}

We evaluate how M2D transfers teacher knowledge from model to data across fair GNNs, Graph Attention Networks (GATs) \citep{velivckovic2018graph}, and Graph Transformers (GTs). Our analysis focuses on how model complexity and fairness properties are distilled into the data. We consider three variants of M2D: feature-only (Feat.), adjacency-only (Adj), and joint feature–adjacency (Both). We additionally evaluate a baseline where the teacher is distilled into the student without M2D (w/o) and a vanilla GCN with no distillation (Van). Both the feature learner and the adjacency learner are implemented as MLPs. Additional details about the experimental setup are in the Appendix.


\subsection{Fair GNN Analysis}
We apply M2D to three fair GNNs (FairGNN, FairVGNN, and NIFTY) \citep{dai2021say,wang2022improving,agarwal2021towards} as teacher models, with a GCN as the student. We evaluate utility via accuracy and fairness using Demographic Parity (DP) and Equal Opportunity (EQOP). Details on datasets, baselines, hyperparameters, and evaluation are provided in the Appendix.

Figure \ref{fig:fairness_accu} shows the accuracy–fairness trade-off for each model and its distilled variants. The y-axis reports accuracy, while the x-axis shows $1\!-\!\text{DP}$ and $1\!-\!\text{EQOP}$. Points closer to the top-right indicate a better trade-off. In the \textsc{NBA} dataset, bias is mainly feature-driven, as fairness improves with feature-only distillation. In contrast, the \textsc{simulation} and the \textsc{German} datasets exhibit bias in both features and structure, requiring joint feature–adjacency distillation.


Figure \ref{fig:fair_feat} shows Kernel Density Estimation (KDE) plots for six features from the \textsc{\textsc{German}} dataset, with FairGNN as the teacher and GCN (Feat.) as the student. The first row shows the original features $\mathbf{X}$ and the second row shows the corresponding learned features $\tilde{\mathbf{X}}$. We select three features that exhibit substantial changes after masking and three that remain unchanged. For the modified features, values concentrated around zero indicate that their contributions are suppressed by the learned masks, illustrating how M2D selectively modifies feature contributions to transfer the teacher's fairness behavior. Analysis of the remaining features and the modified adjacency for the GCN (Adj) student is provided in the Appendix. 
\begin{figure}[t!]
    \centering
    \includegraphics[width=0.9\textwidth]{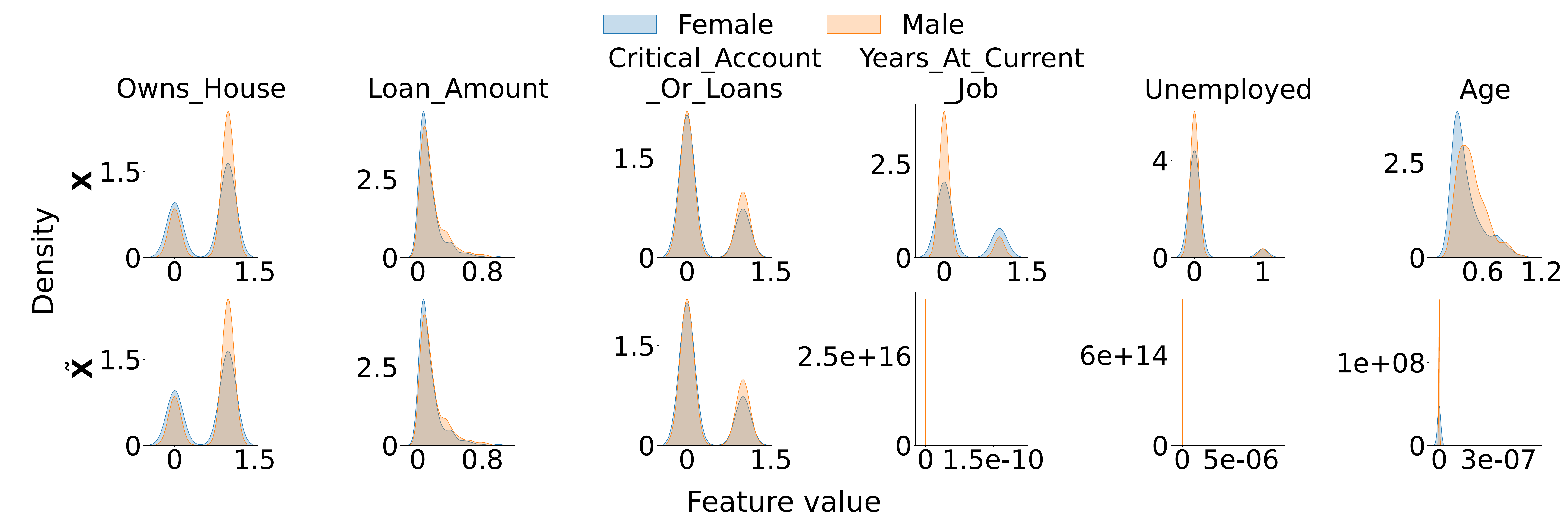}
    \caption{KDE plots for six original ($\mathbf{X}$) and learned ($\tilde{\mathbf{X}}$) features across sensitive groups on the \textsc{German} dataset with GCN (Feat) as the student and FairGNN as the teacher. 
    Features in $\tilde{\mathbf{X}}$ concentrated around zero correspond to those whose contributions are suppressed by the learned masks. }
    \label{fig:fair_feat}
\end{figure}

\subsection{GAT and GT Analysis}
More complex architectures, such as GAT and GTs, are often preferred over simpler models like GCN due to their superior performance. We analyze the advantage of GAT and Graphormer \citep{ying2021transformers} over GCN and MLP (see Appendix) using our approach. Table \ref{tab:gat_gt_acc} reports node classification accuracy across multiple datasets, comparing teacher models (GAT and GT) against the GCN student model under different distillation variants.

\begin{table}[htbp]
\centering
\setlength{\tabcolsep}{3pt} 
\resizebox{0.9\textwidth}{!}{
\begin{tabular}{ll|l||lllll}
\hline
  & & Teacher & GCN (Van.) & GCN (w/o) & GCN (Feat.) & GCN (Adj) & GCN (Both) \\ \hline
\multirow{2}{*}{\textsc{Cora}} & GAT & 83.24 $\pm$ 0.28 & \multirow{2}{*}{79.23 $\pm$ 0.26} & 79.27 $\pm$ 0.13 & \textbf{82.52 $\pm$ 0.44} & 81.10 $\pm$ 0.40 & 82.05 $\pm$ 0.41 \\
  & GT & 77.42 $\pm$ 0.12 & & 76.19 $\pm$ 0.14 & 77.13 $\pm$ 0.21 & 76.14 $\pm$ 0.12 & 76.01 $\pm$ 0.12 \\ \hline
\multirow{2}{*}{\textsc{Citeseer}} & GAT & 73.21 $\pm$ 0.21 & \multirow{2}{*}{71.64 $\pm$ 0.16} & 71.62 $\pm$ 0.27 & \textbf{72.04 $\pm$ 0.13} & 71.98 $\pm$ 0.13 & 72.02 $\pm$ 0.12 \\
  & GT & 65.28 $\pm$ 0.11 & & 64.93 $\pm$ 0.21 & 66.20 $\pm$ 0.21 & 65.11 $\pm$ 0.19 & 65.02 $\pm$ 0.11 \\ \hline
\multirow{2}{*}{\textsc{Photo}} & GAT & 91.35 $\pm$ 0.32 & \multirow{2}{*}{90.23 $\pm$ 0.23} & 90.01 $\pm$ 0.21 & \textbf{90.91 $\pm$ 0.21} & 90.42 $\pm$ 0.20 & 91.27 $\pm$ 0.19 \\
  & GT & 94.20 $\pm$ 0.16 & & 90.08 $\pm$ 0.14 & 92.88 $\pm$ 0.14 & 92.17 $\pm$ 0.18 & \textbf{92.89 $\pm$ 0.11} \\ \hline
\multirow{2}{*}{\textsc{Cornell}} & GAT & 53.67 $\pm$ 0.41 & \multirow{2}{*}{53.61 $\pm$ 0.41} & 53.15 $\pm$ 0.19 & 54.10 $\pm$ 0.21 & 53.64 $\pm$ 0.20 & \textbf{54.13 $\pm$ 0.11} \\
  & GT & 70.59 $\pm$ 0.12 & & 52.98 $\pm$ 0.21 & 58.11 $\pm$ 0.18 & 56.21 $\pm$ 0.11 & \textbf{58.79 $\pm$ 0.11} \\ \hline
\multirow{2}{*}{\textsc{Texas}} & GAT & 64.31 $\pm$ 0.24 & \multirow{2}{*}{62.36 $\pm$ 0.20} & 62.71 $\pm$ 0.13 & 62.69 $\pm$ 0.12 & 64.32 $\pm$ 0.14 & \textbf{64.49 $\pm$ 0.15} \\
  & GT & 73.89 $\pm$ 0.14 & & 62.39 $\pm$ 0.31 & 65.16 $\pm$ 0.11 & 65.59 $\pm$ 0.11 & \textbf{65.56 $\pm$ 0.11} \\ \hline
\end{tabular}}

\caption{Node classification accuracy comparing teacher models (GAT, GT) against a student GCN under different settings. `Van' is the student without any distillation, `w/o' is the standard knowledge distillation (without M2D), and `Feat.', `Adj' and `Both' are M2D applied to features, adjacency, and both, respectively. Results show that M2D, particularly Feat. and Both consistently improve student performance and can match teacher accuracy across datasets.}
\label{tab:gat_gt_acc}
\end{table}


M2D generally improves transfer when the teacher provides useful knowledge, while negative transfer can occur when the teacher is weaker than the student. Feature-based distillation provides the most reliable gains, while adjacency-only distillation is less effective. Joint distillation (both) typically achieves the second-best results, in some cases matching the teacher. These findings demonstrate that M2D can to effectively transfer knowledge from more complex architectures to the data.

Motivated by Theorem \ref{the:feature_weight}, we empirically investigate whether the teacher’s attention patterns are reflected in the similarity structure of the learned feature masks. Figure \ref{fig:gat_feat} analyzes how the learned feature mask aligns with the attention mechanism of a teacher GAT model. We first compute pairwise similarity between node representations using Gaussian similarity based on their Euclidean distance, denoted by, $\mathrm{sim}(\mathbf{m}_i,\mathbf{m}_j)$ and group node pairs into bins based on their similarity. For each bin, we measure the average attention weight of the teacher's existing edges. For homophilic datasets (\textsc{Cora}, \textsc{Citeseer}, \textsc{Photo}), we observe a clear increasing trend: higher similarity corresponds to larger attention. In heterophilic datasets (\textsc{Cornell}, \textsc{Texas}), a weaker but still positive relationship is observed. This difference highlights that the alignment between feature-mask similarity and attention depends on the underlying graph. The analysis of the learned feature mask with respect to the GT attention is in the Appendix. 
\begin{figure}[htbp]
    \centering
    \includegraphics[width=\textwidth]{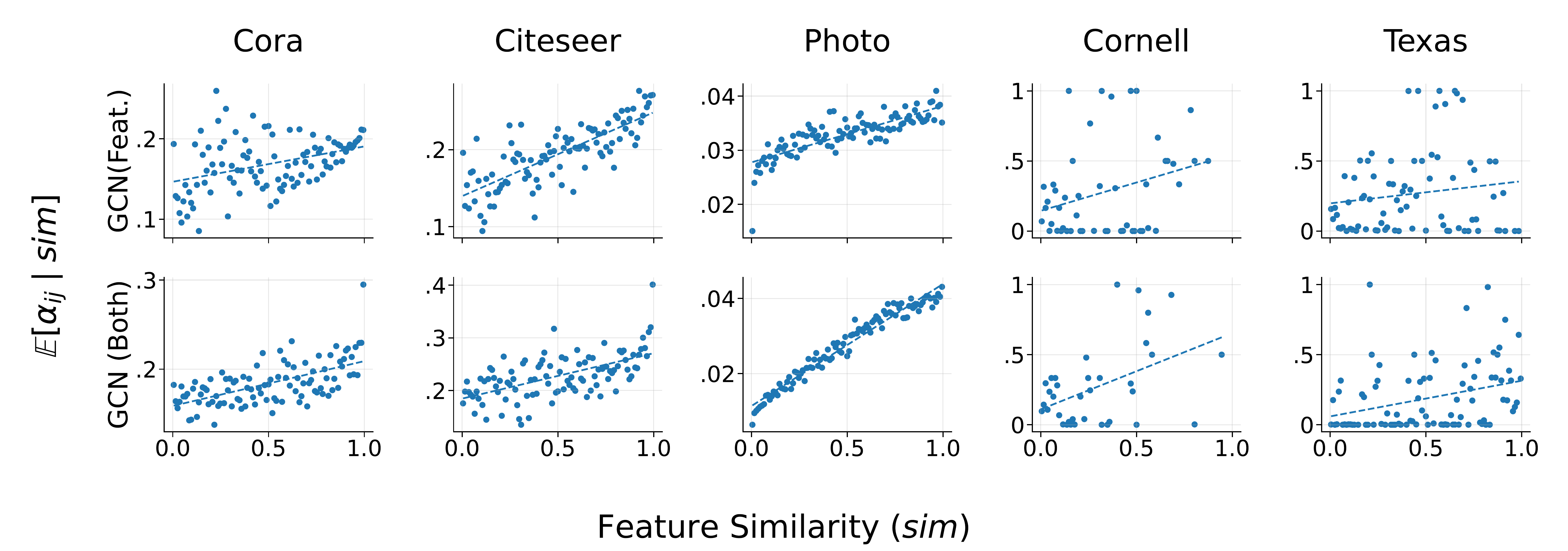}
    \caption{Analysis of the learned feature mask $\mathbf{M}$ with respect to the attention of the teacher GAT. For each dataset, we partition the range of mask similarity $\mathrm{sim}(\mathbf{m}_i, \mathbf{m}_j)$ for existing edges into bins and compute the average attention within each bin. 
    The increasing trend in homophilic datasets (\textsc{Cora}, \textsc{Citeseer}, \textsc{Photo}) indicates that higher attention weights correspond to higher mask similarity. We can also see the trend in heterophilic datasets (\textsc{Cornell}, \textsc{Texas}).}
    \label{fig:gat_feat}
\end{figure}

Figure \ref{fig:gat_adj} visualizes edge importance on \textsc{Cora} (homophilic) using row-normalized M2D edge weights and GAT attention coefficients (Layer 1). Edge color and line width reflect magnitude (darker, thicker edges indicating higher values). The M2D weights are normalized per node (rows sum to one), enabling a comparison with attention coefficients. Using an identical layout, both visualizations exhibit highly similar patterns, indicating that the learned edge weights closely match the GAT attention mechanism. Corresponding results on the heterophilic \textsc{Cornell} dataset, together with a sensitivity analysis of the magnitude of graph modifications, are provided in the Appendix. The sensitivity analysis indicates that M2D can substantially reduce the distillation loss while introducing only limited changes to the original node features and graph structure.

\subsection{Generalization to Multiple Downstream Architectures}
We next investigate whether the behavior encoded in the distilled graph is preserved across downstream architectures. We evaluate whether a single distilled graph enables different models, trained independently from scratch, to maintain the teacher's behavior. Downstream models with distinct architectures are considered: GCN \citep{kipf2016semi}, SGC \citep{wu2019simplifying}, and MLP. Each model is trained using only the distilled graph generated with a GCN as the M2D student and training labels, without further access to the teacher. We compare the resulting models trained on the original and distilled graphs in terms of accuracy, DP, EQOP, and prediction fidelity to the teacher.

\begin{figure}[htbp]
    \centering
    \includegraphics[width=0.78\textwidth]{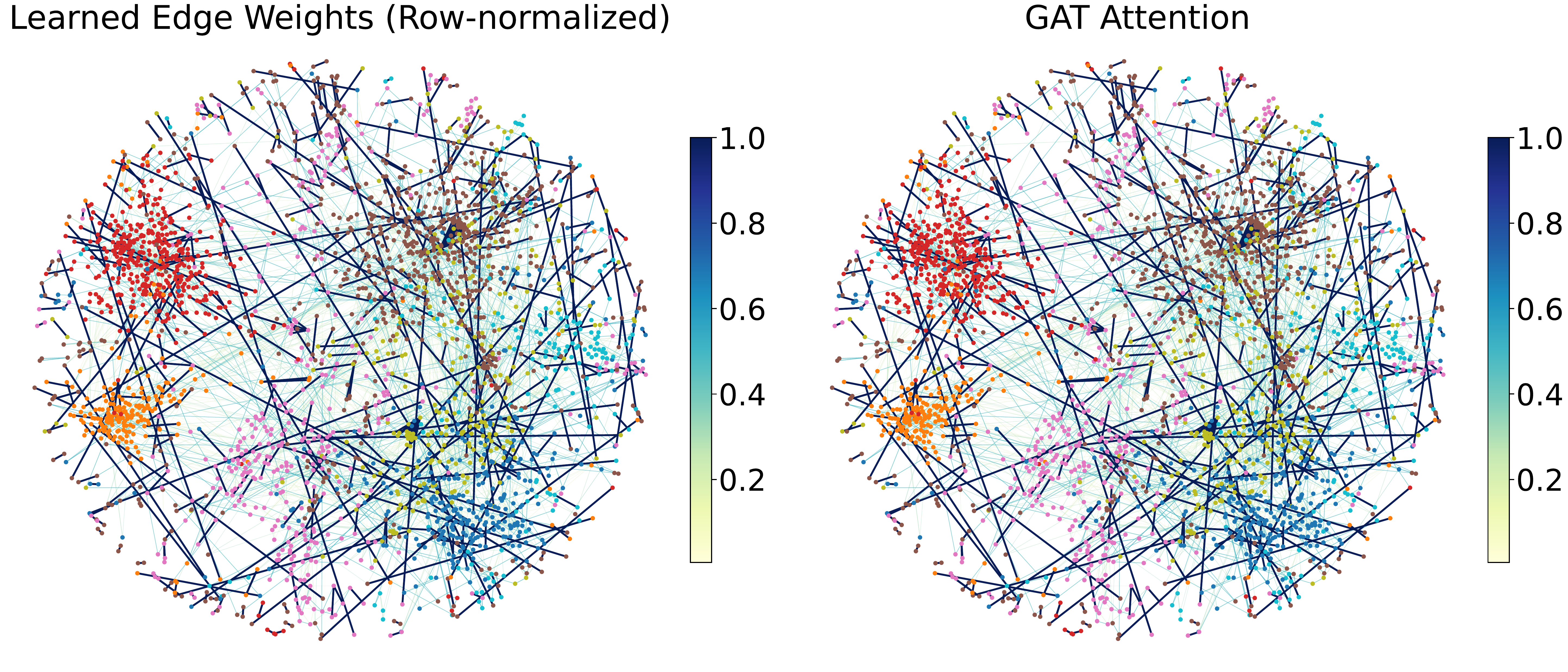}
    \caption{Visualization of row-normalized M2D learned weights (left) and layer-1 GAT attention coefficients (right) for \textsc{Cora}. Edge color intensity and width are proportional to the magnitude of the weight (darker and thicker edges indicate higher importance). Node colors indicate class labels. 
    The learned weights approximate the GAT coefficients (Pearson's correlation = $0.84$).}
    \label{fig:gat_adj}
\end{figure}

\begin{table}[htbp]
\centering
\resizebox{0.9\textwidth}{!}{
\begin{tabular}{l|ll|ll|ll}
        & \multicolumn{2}{l|}{GCN} & \multicolumn{2}{l|}{SGC} & \multicolumn{2}{l}{MLP} \\ \cline{2-7} 
        & Accuracy ($\uparrow$)   & Fidelity ($\uparrow$)  & Accuracy ($\uparrow$)   & Fidelity ($\uparrow$)  & Accuracy ($\uparrow$)  & Fidelity ($\uparrow$)  \\ \cline{2-7} 
Vanilla & 79.23 $\pm$ 0.26  & 91.45 $\pm$ 0.52 & 78.91 $\pm$ 0.42  & 91.23 $\pm$ 0.61 & 63.76 $\pm$ 0.11 & 65.15 $\pm$ 0.29 \\ 
$\tilde{\mathcal G}$(Feat.) & 81.54 $\pm$ 0.41  & 95.21 $\pm$ 0.35 & \textbf{80.97 $\pm$ 0.29}  & 93.54 $\pm$ 0.50 & \textbf{65.06 $\pm$ 0.31} & \textbf{68.29 $\pm$ 0.25} \\ 
$\tilde{\mathcal G}$(Adj)  & 80.61 $\pm$ 0.19  & 94.19 $\pm$ 0.24 & 79.52 $\pm$ 0.41  & 94.01 $\pm$ 0.42 &      -      &     -       \\ 
$\tilde{\mathcal G}$(Both) & \textbf{81.98 $\pm$ 0.27}  & \textbf{95.98 $\pm$ 0.52} & 80.19 $\pm$ 0.30  & \textbf{94.60 $\pm$ 0.13} & - & -
\end{tabular}}
\caption{Generalization of M2D with a GAT teacher on \textsc{Cora} using different downstream architectures trained only on the augmented graph. We evaluate M2D with feature and edge augmentation individually and jointly across GCN, SGC, and MLP. Fidelity consistently increases with distillation, indicating improved alignment between the downstream models and teacher predictions.}
\label{tab:modularity_GAT}
\end{table}

Tables~\ref{tab:modularity_GAT} and~\ref{tab:modularFair_gcn_sgc} evaluate M2D on \textsc{Cora} with a GAT teacher and \textsc{NBA} with a FairGNN teacher, using independently trained GCN, SGC, and MLP as downstream models, respectively. Across both datasets, distillation consistently improves student fidelity to the teacher while preserving comparable predictive accuracy. On \textsc{NBA}, M2D also reduces DP and EQ for most student models, demonstrating that the fairness behavior of the FairGNN teacher is preserved across downstream architectures, including those those different from the student architecture. These results further support our hypothesis that the distilled graph preserves the behavior encoded by the teacher. Results for additional datasets are provided in the Appendix.

\begin{table}[htbp]
\setlength{\tabcolsep}{3pt}
\resizebox{\textwidth}{!}{
\begin{tabular}{l|llll|llll}
         & \multicolumn{4}{l|}{GCN} & \multicolumn{4}{l}{SGC} \\ \cline{2-9}
         & Accuracy ($\uparrow$) & Fidelity ($\uparrow$) & DP ($\downarrow$) & EQOP ($\downarrow$)
         & Accuracy ($\uparrow$) & Fidelity ($\uparrow$) & DP ($\downarrow$) & EQOP ($\downarrow$) \\ \cline{2-9}
Vanilla
& $72.23 \pm 0.42$ & $87.21 \pm 0.33$ & $8.10 \pm 0.51$ & $22.19 \pm 0.62$
& $71.23 \pm 0.16$ & $86.84 \pm 0.41$ & $9.10 \pm 0.52$ & $22.58 \pm 0.53$ \\

$\tilde{\mathcal G}$(Feat.)
& $71.15 \pm 0.29$ & \textbf{94.12 $\pm$ 0.31} & \textbf{6.14 $\pm$ 0.23} & $16.29 \pm 0.23$
& \textbf{70.67 $\pm$ 0.24} & $89.91 \pm 0.53$ & $7.04 \pm 0.30$ & $18.24 \pm 0.24$ \\

$\tilde{\mathcal G}$(Adj)
& $71.59 \pm 0.37$ & $92.97 \pm 0.24$ & $6.97 \pm 0.28$ & \textbf{15.02 $\pm$ 0.42}
& $70.12 \pm 0.31$ & $90.05 \pm 0.24$ & \textbf{6.52 $\pm$ 0.25} & \textbf{18.01 $\pm$ 0.62} \\

$\tilde{\mathcal G}$(Both)
& \textbf{71.81 $\pm$ 0.43} & $93.05 \pm 0.46$ & $6.54 \pm 0.19$ & $17.02 \pm 0.17$
& $70.06 \pm 0.28$ & \textbf{90.67 $\pm$ 0.31} & $6.91 \pm 0.29$ & $18.06 \pm 0.41$
\end{tabular}}
\caption{Generalization of M2D with a FairGNN teacher on \textsc{NBA} using GCN and SGC as downstream architectures trained only on the augmented graph. Fidelity consistently increases with distillation, indicating improved alignment between student and teacher predictions, while fairness is generally improved, and predictive performance remains comparable.}
\label{tab:modularFair_gcn_sgc}
\end{table}

%% file: files/analysis.tex



%% file: files/conclusion.tex
\section{Conclusion}

We introduced Model-to-Data (M2D) distillation, a framework for transferring properties from sophisticated GNNs into graph data rather than into a particular student model. By modifying node features and graph structure, M2D enables simpler models trained with standard objectives to recover behaviors of more complex teachers, effectively trading model complexity for data complexity. Our experiments show that M2D can transfer properties from fairness-aware GNNs, Graph Attention Networks, and Graph Transformers to simpler downstream models, demonstrating that the distilled graph can serve as a reusable medium for preserving teacher behavior beyond the original model. Beyond distillation, the learned data transformations also provide insight into how to modify feature and structural information to achieve the teacher's behavior.


%% file: files/appendix.tex
\appendix
\section{Technical appendices and supplementary material}

This section contains the following:
\begin{itemize}
    \item Background and evaluation metrics
    \item Dataset and baseline descriptions
    \item Proofs of Theorems 1 and 2
    \item M2D training algorithm and experimental setup
    \item Additional results and visualizations
    \item Sensitivity and downstream architecture generalization analyses
    \item Computational complexity and reproducibility details
\end{itemize}

\section*{Graph Neural Networks}
We adopt GCN and SGC  as the students and GAT as the teacher in several experiments~\cite{kipf2016semi, velivckovic2018graph, wu2019simplifying}. All models belong to the class of message-passing neural networks, which learn node representations $\mathbf{h}_v$ used for prediction $y_v = f(\mathbf{h}_v)$. These representations are computed via iterative $\mathsf{AGGREGATE}$ and $\mathsf{COMBINE}$ operations:
\begin{equation*}
\mathbf{h}_v^{t+1} = \mathsf{COMBINE}^t\Big(\mathbf{h}_v^{t}, \; \mathsf{AGGREGATE}^t\big(\{\mathbf{h}_u^{t} : A_{uv} = 1\}\big)\Big),
\end{equation*}
where $\mathbf{h}_v^{t}$ denotes the embedding of the node $v$ at layer $t$. 

The primary difference between GCN and GAT lies in the aggregation mechanism. GCN employs fixed, normalized adjacency weights to average neighbor features, whereas GAT learns attention coefficients that assign different importance to neighbors, enabling adaptive and data-dependent message passing.

For GCN, the layer-wise update is given by:
\begin{equation*}
\mathbf{h}_v^{t+1} = \sigma\big(
\frac{1}{\sqrt{|N(v)|+1}} \mathbf{W}^t
\big(
\frac{\mathbf{h}_v^{t}}{\sqrt{|N(v)|+1}} +
\sum_{u \in N(v)} \frac{\mathbf{h}_u^{t}}{\sqrt{|N(u)|+1}}
\big)
\big),
\end{equation*}

For GAT, the update rule is:
\begin{equation*}
\mathbf{h}_v^{t+1} = \sigma\big(
\sum_{u \in N(v)\cup\{v\}} \alpha_{vu} \,\mathbf{W} \mathbf{h}_u^{t}
\big),
\end{equation*}
where $\alpha_{vu}$ denotes the attention coefficient on edge $(v,u)$:
\begin{equation*}
\alpha_{vu} =
\frac{
\exp\Big(\mathrm{LeakyReLU}\big(\mathbf{a}^\top
[\mathbf{h}_v^{t}\mathbf{W}, \; \mathbf{h}_u^{t}\mathbf{W}]
\big)\Big)
}{
\sum_{k \in N(v)\cup\{v\}}
\exp\Big(\mathrm{LeakyReLU}\big(\mathbf{a}^\top
[\mathbf{h}_v^{t}\mathbf{W}, \; \mathbf{h}_k^{t}\mathbf{W}]
\big)\Big)
},
\end{equation*}
with learnable parameters $\mathbf{W}$ and $\mathbf{a}$.

SGC simplifies the GCN architecture by removing intermediate nonlinearities and collapsing multiple propagation steps into a single linear transformation. Given the normalized adjacency matrix $\hat{\mathbf{A}}$, the node representations after $K$ propagation steps are computed as:
\begin{equation*}
\mathbf{H}^{(K)} = \hat{\mathbf{A}}^{K}\mathbf{X},
\end{equation*}
where $\mathbf{X}$ denotes the input node features. The final node representations are then passed through a linear classifier:
\begin{equation*}
\mathbf{Y} = \mathbf{H}^{(K)}\mathbf{W},
\end{equation*}
where $\mathbf{W}$ is a learnable weight matrix. By removing intermediate nonlinear transformations, SGC provides a simpler downstream architecture while retaining multi-hop neighborhood information through repeated graph propagation.

We also consider a \textbf{Graph Transformer (GT)} as a teacher model. More specifically, we consider Graphormer~\cite{ying2021transformers} as a representative GT. GTs are built by stacking multiple Transformer layers~\cite{vaswani2017attention}, where each layer comprises a self-attention module followed by a feed-forward network.

Let $\mathbf{H} \in \mathbb{R}^{n \times d}$ denote the input to the self-attention module. This input is linearly projected into query, key, and value representations via learnable matrices $\mathbf{W}_Q$, $\mathbf{W}_K$, and $\mathbf{W}_V$:
\begin{align*}
\mathbf{Q} = \mathbf{H}\mathbf{W}_Q, \quad
\mathbf{K} = \mathbf{H}\mathbf{W}_K, \quad
\mathbf{V} = \mathbf{H}\mathbf{W}_V.
\end{align*}

The self-attention operation is then computed as:
\begin{align*}
\mathrm{Attention}(\mathbf{Q}, \mathbf{K}, \mathbf{V}) =
\mathrm{softmax}\big(
\frac{\mathbf{Q}\mathbf{K}^\top}{\sqrt{d_K}}
\big)\mathbf{V},
\end{align*}
where $d_K$ denotes the dimensionality of the key vectors.

\section*{Evaluation Metrics}

\textbf{Utility:}
We use accuracy to evaluate utility.

\textbf{Fairness:}
We measure group fairness with respect to a sensitive attribute $s$, where $a=1$ denotes the sensitive group and $a=0$ the non-sensitive group. Lower values indicate better fairness.

\textit{Demographic Parity (DP)} measures the difference in positive prediction rates across groups:
\begin{equation}
DP(\hat{Y}, a) = \big| p(\hat{Y}=1 \mid a=1) - p(\hat{Y}=1 \mid a=0) \big|.
\end{equation}

\textit{Equality of Opportunity (EQOP)} measures the difference in true positive rates:
\begin{equation}
EQOP(\hat{Y}, a, y) = \big| p(\hat{Y}=1 \mid a=1, y=1) - p(\hat{Y}=1 \mid a=0, y=1) \big|.
\end{equation}
DP and EQOP both measure bias. We can get corresponding fairness by subtracting them from 1.
\section*{Dataset Details}
\subsubsection*{Dataset for Fairness Evaluation:}
\textsc{Simulation}:
We generate a synthetic graph with 1000 nodes using a stochastic block model with two communities of sizes 600 (majority) and 400 (minority), and set intra and inter-block edge probabilities to achieve high assortativity ($\approx 0.77$). The sensitive attribute $a$ is generated conditionally on the label $\mathbf{y}$: $a \sim \text{Bernoulli}(p)$ if $y=1$ and $a \sim \text{Bernoulli}(1-p)$ otherwise, we set $p=0.7$ to induce bias, creating high correlation between $a$ and $\mathbf{y}$. This setup introduces structural bias, as nodes are more likely to connect within the same community, leading to label–group correlations in the graph topology. Each node has 20 features, of which 8 are noisy, sampled from $\mathcal{N}(0,1)$, while the remaining features are generated as $\gamma \cdot y + \mathcal{N}(0,1)$. This introduces feature-level bias, as informative attributes are correlated with the label and, indirectly, with the sensitive attribute, reinforcing disparities present in both structure and features.

\textsc{German} \citep{German}:
Vertices represent 1000 individuals with 20 financial and demographic features. Edges are computed from vertex similarity. The prediction task is to classify credit (good vs. bad credit), with gender treated as the sensitive attribute.

\textsc{NBA} \citep{dai2021say}:
This dataset consists of 403 NBA players connected via their Twitter interactions. Node features capture player statistics, and the task is to predict whether a player’s salary is above or below the median. Nationality is treated as a sensitive attribute.
\subsubsection*{Dataset for Improved Performance Evaluation}
Statistics for the datasets used to analyze the performance of GAT and GT are provided in Table \ref{tab:data_gat}.

\section*{Baselines}
Vanilla: The vanilla model is the student model without any distillation. 

FairGNN \cite{dai2021say}: Jointly trains a GNN classifier with a neural network adversary, where the adversary encourages the learned representations to remove sensitive information and thus mitigate bias.

FairVGNN \cite{wang2022improving}: FairVGNN learns fair feature views by identifying and masking features correlated with sensitive attributes, explicitly accounting for shifts in these correlations induced by feature propagation. Conditioned on the resulting views, it adaptively regularizes the encoder parameters to suppress the influence of sensitive-related information.

NIFTY \cite{agarwal2021towards}: NIFTY ensures fairness and stability by combining a joint fairness–stability objective with Lipschitz-based layer-wise weight normalization. It provides theoretical guarantees for counterfactual fairness and stable representations, and achieves strong empirical performance on high-stakes datasets.

GT \cite{ying2021transformers}: Transformer-based model for graph representation learning that achieves strong performance across a wide range of tasks. Its core design centers on explicitly encoding graph structure into the Transformer architecture through centrality encoding, edge encoding, and spatial encoding, enabling effective modeling of node relationships beyond standard message passing.

\begin{table}[htbp]
\centering
\begin{tabular}{l|l|l|l|l|l}
Dataset  & Type         & Nodes & Edges   & Classes & Features \\ \hline
\textsc{Cora}     & Homophilic   & 2708  & 5429    & 7       & 1433     \\
\textsc{Citeseer} & Homophilic   & 3327  & 4732    & 6       & 3703     \\
\textsc{Photo}    & Homophilic   & 7650 & 238163 & 8       & 745      \\
\textsc{Cornell}  & Heterophilic & 183   & 298     & 5       & 1703     \\
\textsc{Texas}    & Heterophilic & 183   & 325     & 5       & 1703 \\ \hline
\end{tabular}
\caption{Dataset statistics \cite{sen2008collective,shchur2018pitfalls,garcia2016using}}
\label{tab:data_gat}
\end{table}
\section*{Proof of Theorems}

Our theorems are based on feature augmentations. A similar analysis can be provided for structural augmentations where the learned structure approximates the representations from the teacher model.

\subsection*{Theorem 1}
\begin{proof}
The demographic parity of the student $f_s$ is
\begin{equation*}
\mathrm{DP}(f_s)=\big|\mathbb{E}[\hat{Y}_s\mid a=0]-\mathbb{E}[\hat{Y}_s\mid a=1]\big|.
\end{equation*}

Adding and subtracting the corresponding teacher group means gives
\begin{align*}
\mathrm{DP}(f_s)=&\big|\mathbb{E}[\hat{Y}_s\mid a=0]-\mathbb{E}[\hat{Y}_T\mid a=0]
\\
&+\mathbb{E}[\hat{Y}_T\mid a=0]-\mathbb{E}[\hat{Y}_T\mid a=1]
\\
&+\mathbb{E}[\hat{Y}_T\mid a=1]-\mathbb{E}[\hat{Y}_s\mid a=1]\big|.
\end{align*}

Applying the triangle inequality,
\begin{align*}
\mathrm{DP}(f_s)\leq\;&\mathrm{DP}(f_T)
\\
&+\big|\mathbb{E}[\hat{Y}_s-\hat{Y}_T\mid a=0]\big|
\\
&+\big|\mathbb{E}[\hat{Y}_s-\hat{Y}_T\mid a=1]\big|.
\end{align*}

For each group $g\in\{0,1\}$,
\begin{equation*}
\Big|\mathbb{E}[\hat{Y}_s-\hat{Y}_T\mid a=g]\Big|\leq \mathbb{E}\Big[|\hat{Y}_s-\hat{Y}_T|\Big| a=g\Big].
\end{equation*}

For Bernoulli distributions, the total-variation distance satisfies
\begin{equation*}
D_{\mathrm{TV}}\big(\mathrm{Bern}(\hat{Y}_T),\mathrm{Bern}(\hat{Y}_s)\big)=|\hat{Y}_T-\hat{Y}_s|.
\end{equation*}

By Pinsker's inequality,
\begin{equation*}
|\hat{Y}_T-\hat{Y}_s|\leq\sqrt{\frac{1}{2}D_{\mathrm{KL}}
\big(\mathrm{Bern}(\hat{Y}_T),\mathrm{Bern}(\hat{Y}_s)\big)}.
\end{equation*}

Taking the conditional expectation over group $g$ gives
\begin{align*}
\mathbb{E}
\big[
|\hat{Y}_T-\hat{Y}_s|
\mid a=g
\big]
&\leq
\mathbb{E}
\big[
\sqrt{
\frac{1}{2}
D_{\mathrm{KL}}
\big(
\mathrm{Bern}(\hat{Y}_T),
\mathrm{Bern}(\hat{Y}_s)
\big)}
\mid a=g
\big].
\end{align*}

Since the square-root function is concave, Jensen's inequality yields
\begin{align*}
\mathbb{E}
\Big[
|\hat{Y}_T-\hat{Y}_s|
\Big| a=g
\Big]
&\leq
\sqrt{
\frac{1}{2}
\mathbb{E}
\big[
D_{\mathrm{KL}}
\big(
\mathrm{Bern}(\hat{Y}_T),
\mathrm{Bern}(\hat{Y}_s)
\big)
\mid a=g
\big]
}
\\
&=
\sqrt{\frac{\kappa_g}{2}}.
\end{align*}
where, $\kappa_g = \mathbb{E}
\big[D_{\mathrm{KL}}\big(\mathrm{Bern}(\hat{Y}_T),\mathrm{Bern}(\hat{Y}_s)\big)\big | a=g \big]$

Therefore,
\begin{equation*}
\mathrm{DP}(f_s)
\leq
\mathrm{DP}(f_T)
+
\sqrt{\frac{\kappa_0}{2}}
+
\sqrt{\frac{\kappa_1}{2}}.
\end{equation*}

This bound depends only on the final teacher-student prediction discrepancy. We next characterize the variation between the M2D student $f_s$ and a newly trained downstream model $f_{s'}$ on the frozen graph $\tilde{\mathcal{G}}$. Conditioned on an input $\tilde{\mathcal{G}}$, let
$\eta=\Pr({Y}\mid\tilde{\mathcal{G}})$ be the Bayes optimal conditional probability of class label ${Y}$.

Applying Pinsker's inequality pointwise gives
\begin{equation*}
|\hat{Y}_{s'}-\eta|
\leq
\sqrt{
\frac{1}{2}
D_{\mathrm{KL}}
\big(
\mathrm{Bern}(\eta)
,
\mathrm{Bern}(\hat{Y}_{s'})
\big)
}.
\end{equation*}

Taking the conditional expectation and applying Jensen's inequality,
\begin{align*}
\mathbb{E}
\Big[
|\hat{Y}_{s'}-\eta|
\Big | a=g
\Big]
&\leq
\sqrt{
\frac{1}{2}
\mathbb{E}
\big[
D_{\mathrm{KL}}
\big(
\mathrm{Bern}(\eta),
\mathrm{Bern}(\hat{Y}_{s'})
\big)
\big | a=g
\big]
}
\\
&=
\sqrt{\frac{\xi_g(f_{s'})}{2}}.
\end{align*}

Now consider $f_s$ and $f_{s'}$.
\begin{equation*}
|\hat{Y}_s-\hat{Y}_{s'}|
=
|(\hat{Y}_s-\eta)+(\eta-\hat{Y}_{s'})|
\leq
|\hat{Y}_s-\eta|
+
|\eta-\hat{Y}_{s'}|
=
|\hat{Y}_s-\eta|
+
|\hat{Y}_{s'} - \eta|.
\end{equation*}

Taking the conditional expectation within group $g$,
\begin{align*}
\mathbb{E}
\Big[
|\hat{Y}_s-\hat{Y}_{s'}|
\Big | a=g
\Big]
&\leq
\mathbb{E}
\Big[
|\hat{Y}_s-\eta|
\Big | a=g
\Big]
+
\mathbb{E}
\Big[
|\hat{Y}_{s'}-\eta|
\Big | a=g
\Big]
\\
&\leq
\sqrt{\frac{\xi_g(f_s)}{2}}
+
\sqrt{\frac{\xi_g(f_{s'})}{2}}.
\end{align*}
Similarly,
\begin{equation*}
|\hat{Y}_{s'}-\hat{Y}_{T}|
=
|(\hat{Y}_{s'}-\eta)+(\eta-\hat{Y}_{s})+(\hat{Y}_{s}-\hat{Y}_{T})|
\leq
|\hat{Y}_{s'}-\eta|+|\hat{Y}_{s}-\eta|+|\hat{Y}_{s}-\hat{Y}_{T}|.
\end{equation*}

Therefore,
\begin{align*}
\mathbb{E}
\Big[
|\hat{Y}_{s'}-\hat{Y}_T|
\Big | a=g
\Big]
\leq\;&
\sqrt{\frac{\xi_g(f_s)}{2}}
+
\sqrt{\frac{\xi_g(f_{s'})}{2}}
+
\sqrt{\frac{\kappa_g}{2}}.
\end{align*}

Applying the same demographic-parity decomposition used above to the
downstream student $f_{s'}$ yields
\begin{align*}
\mathrm{DP}(f_{s'})
\leq\;&
\mathrm{DP}(f_T)
\\
&+
\mathbb{E}
\big[
|\hat{Y}_{s'}-\hat{Y}_T|
\mid a=0
\big]
\\
&+
\mathbb{E}
\big[
|\hat{Y}_{s'}-\hat{Y}_T|
\mid a=1
\big].
\end{align*}

Substituting the previous bound for each sensitive group gives
\begin{align*}
\mathrm{DP}(f_{s'})
\leq\;&
\mathrm{DP}(f_T)
+
\sum_{g\in\{0,1\}}
\left(
\sqrt{\frac{\kappa_g}{2}} +
\big(
\sqrt{\frac{\xi_g(f_s)}{2}}
+
\sqrt{\frac{\xi_g(f_{s'})}{2}}
\right).
\end{align*} 
Therefore,
\begin{align*}
\mathrm{DP}(f_{s'}) - \mathrm{DP}(f_T)
\leq\ 
\sum_{g\in\{0,1\}}
\left(
\sqrt{\frac{\kappa_g}{2}} +
\big(
\sqrt{\frac{\xi_g(f_s)}{2}}
+
\sqrt{\frac{\xi_g(f_{s'})}{2}}
\right).
\end{align*} 

Similarly, we can get

\begin{align*}
\mathrm{DP}(f_{T}) - \mathrm{DP}(f_{s'})
\leq\ 
\sum_{g\in\{0,1\}}
\left(
\sqrt{\frac{\kappa_g}{2}} +
\big(
\sqrt{\frac{\xi_g(f_s)}{2}}
+
\sqrt{\frac{\xi_g(f_{s'})}{2}}
\right).
\end{align*} 

Hence

\begin{align*}
\left| \mathrm{DP}(f_{T}) - \mathrm{DP}(f_{s'}) \right|
\leq\ 
\sum_{g\in\{0,1\}}
\left(
\sqrt{\frac{\kappa_g}{2}} +
\big(
\sqrt{\frac{\xi_g(f_s)}{2}}
+
\sqrt{\frac{\xi_g(f_{s'})}{2}}
\right).
\end{align*}

Therefore, the fairness difference between a new downstream model and the teacher is controlled by the distillation error and how well the original student and the new downstream model approximate the optimal predictor on the transformed graph, while both the student and the downstream model optimize cross-entropy loss.
\end{proof}

\subsubsection*{Theorem 2}
\begin{proof}
Consider a single-layer GAT with teacher representation
\[
h_{T,i}
=
\sigma\left(
\sum_{j\in\mathcal N(i)}
\alpha_{ij}Wx_j
\right).
\]
Using a first-order Taylor approximation, there exist a linear
map $\widetilde W$ and residual $b_i$ such that
\[
h_{T,i}
=
\sum_{j\in\mathcal N(i)}
\alpha_{ij}\widetilde W x_j+b_i.
\]

Let the teacher and student logits be
\[
z_{T,i}=U_T h_{T,i},
\qquad
z_{s,i}=U_s h_{s,i},
\]
and define their temperature-scaled predictive distributions as
\[
p_{T,i}=\operatorname{softmax}(z_{T,i}/\tau),
\qquad
p_{s,i}=\operatorname{softmax}(z_{s,i}/\tau).
\]

Let
\[
R=
\begin{bmatrix}
1 & 0 & \cdots & 0 & -1\\
0 & 1 & \cdots & 0 & -1\\
\vdots & & \ddots & & \vdots\\
0 & 0 & \cdots & 1 & -1
\end{bmatrix}
\in\mathbb{R}^{(C-1)\times C},
\]
so that, for any logit vector $\mathbf z\in\mathbb{R}^{C}$,
\[
R_{\mathbf z}
=
[z_1-z_C,\ldots,z_{C-1}-z_C]^\top.
\]

Since, for any logit vector $\mathbf z$ with
$\mathbf p=\operatorname{softmax}(\mathbf z/\tau)$,
\[
R_\mathbf z
=
\tau
\begin{bmatrix}
\log(p_1/p_C)\\
\vdots\\
\log(p_{C-1}/p_C)
\end{bmatrix},
\]
Since $R$
subtracts the logit of the reference class $C$ from the remaining
class logits,
\[
R_{\mathbf z_{s,i}}
=
\begin{bmatrix}
z_{s,i,1}-z_{s,i,C}\\
\vdots\\
z_{s,i,C-1}-z_{s,i,C}
\end{bmatrix},
\]
and similarly for the teacher.

Therefore,
\[
z_{s,i,c}-z_{s,i,C}
=
\tau
\log\left(
\frac{p_{s,i,c}}{p_{s,i,C}}
\right).
\]
Likewise,
\[
z_{T,i,c}-z_{T,i,C}
=
\tau
\log\left(
\frac{p_{T,i,c}}{p_{T,i,C}}
\right).
\]

We define the contrast-logit discrepancy as:
\[
\boldsymbol{\epsilon}_i
=
R_{\mathbf z_{s,i}}
-
R_{\mathbf z_{T,i}}.
\]
Therefore,
\[
RU_s\mathbf h_{s,i}
=
RU_T\mathbf h_{T,i}
+
\boldsymbol{\epsilon}_i.
\]

Let
\[
\kappa_i
=
D_{\mathrm{KL}}(p_{T,i} ~,~ p_{s,i}).
\]
By Pinsker's inequality,
\[
\|p_{T,i}-p_{s,i}\|_1
\le
\sqrt{2\kappa_i}.
\]
Hence, as $\kappa_i\rightarrow0$,
\[
p_{s,i}\rightarrow p_{T,i}.
\]
Since finite logits imply strictly positive softmax probabilities,
$p_{T,i,c}>0$ for all classes $c$, each function
\[
\log\left(\frac{p_c}{p_C}\right)
\]
is continuous at $\mathbf p_{T,i}$. Therefore, as
$\mathbf p_{s,i}\rightarrow \mathbf p_{T,i}$,
\[
\log\left(\frac{p_{s,i,c}}{p_{s,i,C}}\right)
-
\log\left(\frac{p_{T,i,c}}{p_{T,i,C}}\right)
\rightarrow 0
\]
for every $c=1,\ldots,C-1$. Hence,
\[
\|\boldsymbol{\epsilon}_i\|_2 \rightarrow 0.
\]

Define
\[
\bar U_s=RU_s,
\qquad
\bar U_T=RU_T.
\]
Then
\[
\bar U_s h_{s,i}
=
\bar U_T h_{T,i}+\epsilon_i.
\]

Since multiple student representations can produce the same output through a linear prediction head, we adopt the minimum-\(\ell_2\)-norm solution as a canonical choice, consistent with the implicit bias of zero-initialized gradient descent in underdetermined linear least-squares problems. Therefore,
\[
h_{s,i}
=
\bar U_s^\dagger
\left(
\bar U_T h_{T,i}+\epsilon_i
\right).
\]
Substituting the teacher representation gives
\[
h_{s,i}
=
\bar U_s^\dagger
\left(
\sum_{j\in\mathcal N(i)}
\alpha_{ij}
\bar U_T\widetilde W x_j
+
\bar U_Tb_i+\epsilon_i
\right).
\]

Let
\[
Q
=
\bar U_s^\dagger\bar U_T\widetilde W
\]
and
\[
r_i
=
\bar U_s^\dagger
\left(
\bar U_Tb_i+\epsilon_i
\right).
\]
Then
\[
h_{s,i}
=
\sum_{j\in\mathcal N(i)}
\alpha_{ij}Qx_j+r_i.
\]

Defining
\[
P_i
=
\sum_{j\in\mathcal N(i)}
\alpha_{ij}Qx_j+r_i,
\]
the feature learner yields
\[
m_i
=
\sigma(f_\phi(P_i)).
\]

Similarly, for an edge $(i,j)$,
\[
h_{s,i}+h_{s,j}=P_i+P_j,
\qquad
|h_{s,i}-h_{s,j}|=|P_i-P_j|.
\]
Thus,
\[
w_{ij}
=
\sigma\left(
f_a(P_i+P_j,|P_i-P_j|)
\right).
\]

The residual can be decomposed as
\[
\mathbf r_i
=
\bar U_s^\dagger\bar U_T\mathbf b_i
+
\bar U_s^\dagger\boldsymbol{\epsilon}_i.
\]
Thus,
\[
\|\mathbf r_i\|_2
\le
\|\bar U_s^\dagger\bar U_T\mathbf b_i\|_2
+
\|\bar U_s^\dagger\boldsymbol{\epsilon}_i\|_2.
\]
The first term arises from the local linearization of the teacher
representation, whereas the second term is induced by the
teacher-student prediction discrepancy.
As
$D_{\mathrm{KL}}(\mathbf p_{T,i} ~, ~\mathbf p_{s,i})\rightarrow0$,
we have $\|\boldsymbol{\epsilon}_i\|_2\rightarrow0$.

Thus, both the learned feature masks and edge weights are
measurable transformations of the teacher's attention-weighted neighborhood aggregations, up to residuals arising from local linearization and imperfect distillation.
\end{proof}

\section*{Training Algorithm}
Algorithm \ref{alg:m2d} shows the training of M2D. M2D performs iterative distillation by jointly refining the student model and the underlying graph structure. At each iteration, the current student representation is used to generate an augmented graph. The student is then updated using this augmented input to match the teacher outputs. Subsequently, the graph and feature learners are updated based on the refined student representation. This alternating process enables mutual improvement between representation learning and structure adaptation. The final augmented graph and features capture the teacher’s behavior more effectively.
\begin{algorithm}[t]
\caption{M2D Distillation}
\begin{algorithmic}[1]
\REQUIRE Graph $\mathbf{A}$, features $\mathbf{X}$, labels $\mathbf{y}$, teacher outputs $\hat{\mathbf{y}}_T$
\STATE Initialize $\Theta_s^{0}$, $\Theta_g^{0}$ and $\mathbf{H}_s^{0} \rightarrow f_s(\mathbf{X}, \mathbf{A})$

\FOR{$t = 1$ to $t_{\max}$}

\STATE \textbf{(1) Construct augmented graph and features}
\STATE $\mathbf{M}^{t} \rightarrow f_\phi(\mathbf{H}_s^{t-1})$, \quad $\tilde{\mathbf{X}}^{t} \rightarrow [\mathbf{X} \odot \mathbf{M}^{t}]$
\STATE $\mathbf{W}^{t} \rightarrow f_a(\mathbf{H}_s^{t-1})$
\STATE $\tilde{\mathbf{A}}^{t} \rightarrow (1-\gamma)\mathbf{A} + \gamma\big((1-\beta)\mathbf{W}^{t} + \beta\, s(\mathbf{X})\big)$

\vspace{2pt}
\STATE \textbf{(2) Update student (inner step)}
\STATE $\mathbf{H}_s^{t} \rightarrow f_s(\tilde{\mathbf{X}}^{t}, \tilde{\mathbf{A}}^{t})$
\STATE $\hat{\mathbf{y}}_s^{t} \rightarrow f_c(\mathbf{H}_s^{t})$
\STATE Update $\Theta_s^{t}$ using $\mathcal{L}_\mathrm{cls}$

\vspace{2pt}
\STATE \textbf{(3) Update graph and feature learner (outer step)}
\STATE $\mathbf{M}^{t} \rightarrow f_\phi(\mathbf{H}_s^{t})$, \quad $\mathbf{W}^{t} \rightarrow f_a(\mathbf{H}_s^{t})$
\STATE Update $\Theta_g^{t}$ using 
$\mathcal{L}_\mathrm{dis} + \mathcal{L}_\mathrm{feat} + \mathcal{L}_\mathrm{graph}$

\ENDFOR

\RETURN $\tilde{\mathbf{A}}^{t}$, $\tilde{\mathbf{X}}^{t}$
\end{algorithmic}
\label{alg:m2d}
\end{algorithm}
\section*{Software and Hardware}
\begin{itemize}
    \item Operating System: Linux (Red Hat Enterprise Linux 8.9 (Ootpa))
    \item GPU: NVIDIA A40
    \item Software: Python 3.8.10, torch 2.2.1
\end{itemize}

\section*{Experimental Setup}
All experiments use two-dimensional learned features for all datasets, except \textsc{NBA}, where a single feature is learned. All GNN models consist of two layers. We use a linear prediction head for both student and teacher in all experiments. For Graph Transformer experiments, we adopt Graphormer \cite{ying2021transformers}. All reported results are averaged over $5$ independent runs with different random seeds. Hyperparameters are selected via grid search based on validation performance. To mitigate overfitting, we employ early stopping after $150$ iterations. Each dataset is divided into 30\% for training, 20\% for validation, and 50\% for testing. We use the Adam optimizer to train all model parameters. 
\begin{table}[htbp]
\centering
\setlength{\tabcolsep}{3pt} 

\begin{tabular}{ll|l||lll}
\hline
  & & Teacher & MLP (Van.) & MLP (w/o) & MLP (Feat.) \\ \hline
\multirow{2}{*}{\textsc{Cora}} & GAT & 83.24 $\pm$ 0.28 & \multirow{2}{*}{63.76 $\pm$ 0.11} & 63.85 $\pm$ 0.21 & 64.58 $\pm$ 0.13 \\
  & GT & 77.42 $\pm$ 0.12 & & 63.49 $\pm$ 0.18 & 62.79 $\pm$ 0.24 \\ \hline
\multirow{2}{*}{\textsc{Citeseer}} & GAT & 73.21 $\pm$ 0.21 & \multirow{2}{*}{68.71 $\pm$ 0.14} & 67.98 $\pm$ 0.13 & 69.48 $\pm$ 0.21 \\
  & GT & 65.28 $\pm$ 0.11 & & 66.96 $\pm$ 0.22 & 66.11 $\pm$ 0.14 \\ \hline
\multirow{2}{*}{\textsc{Photo}} & GAT & 91.35 $\pm$ 0.32 & \multirow{2}{*}{85.03 $\pm$ 0.09} & 84.18 $\pm$ 0.17 & 86.94 $\pm$ 0.14 \\
  & GT & 94.20 $\pm$ 0.16 & & 84.95 $\pm$ 0.16 & 88.47 $\pm$ 0.26 \\ \hline
\multirow{2}{*}{\textsc{Cornell}} & GAT & 53.67 $\pm$ 0.41 & \multirow{2}{*}{66.34 $\pm$ 0.18} & 59.31 $\pm$ 0.22 & 62.98 $\pm$ 0.15 \\
  & GT & 70.59 $\pm$ 0.12 & & 64.19 $\pm$ 0.16 & 68.80 $\pm$ 0.22 \\ \hline
\multirow{2}{*}{\textsc{Texas}} & GAT & 64.31 $\pm$ 0.24 & \multirow{2}{*}{75.27 $\pm$ 0.20} & 70.20 $\pm$ 0.23 & 71.11 $\pm$ 0.14 \\
  & GT & 73.89 $\pm$ 0.14 & & 72.15 $\pm$ 0.11 & 74.05 $\pm$ 0.19 \\ \hline
\end{tabular}

\caption{Node classification accuracy comparing teacher models (GAT, GT) and student MLP under different settings. 
`Van.' denotes the student trained without distillation, `w/o' standard distillation, and `Feat.' M2D feature distillation. 
M2D improves student performance on \textsc{Cora}, \textsc{Citeseer}, and \textsc{Photo}. 
On \textsc{Cornell} and \textsc{Texas}, where the teacher underperforms the vanilla MLP, M2D degrades performance by transferring suboptimal knowledge.}
\label{tab:gat_gt_acc_mlp}
\end{table}

\begin{table}[htbp]
\begin{tabular}{l|l|llll}
                            & Teacher  & GCN(Feat.) & GCN(Adj.)  & GCN(Both)  & GCN(w/0)   \\ \cline{1-6} 
\multirow{3}{*}{\textsc{Simulation}} & FairGNN  & 94.12 $\pm$ 0.13 & 91.19 $\pm$ 0.21 & 93.25 $\pm$ 0.19 & 89.14 $\pm$ 0.11 \\
                            & FairVGNN & 96.03 $\pm$ 0.12 & 93.12 $\pm$ 0.26 & 96.91 $\pm$ 0.13 & 92.13 $\pm$ 0.30 \\
                            & NIFTY    & 94.31 $\pm$ 0.12 & 95.12 $\pm$ 0.30 & 95.59 $\pm$ 0.24 & 92.15 $\pm$ 0.21 \\ \hline
\multirow{3}{*}{NBA}        & FairGNN  & 95.13 $\pm$ 0.11 & 92.31 $\pm$ 0.25 & 93.91 $\pm$ 0.24 & 93.63 $\pm$ 0.27 \\
                            & FairVGNN & 94.21 $\pm$ 0.15 & 90.63 $\pm$ 0.34 & 93.01 $\pm$ 0.13 & 92.21 $\pm$ 0.19 \\
                            & NIFTY    & 96.13 $\pm$ 0.12 & 91.79 $\pm$ 0.15 & 94.59 $\pm$ 0.28 & 94.14 $\pm$ 0.25 \\ \hline
\multirow{3}{*}{\textsc{German}}     & FairGNN  & 96.29 $\pm$ 0.21 & 95.97 $\pm$ 0.14 & 96.25 $\pm$ 0.17 & 92.12 $\pm$ 0.25 \\
                            & FairVGNN & 94.98 $\pm$ 0.12 & 95.23 $\pm$ 0.34 & 95.72 $\pm$ 0.31 & 91.79 $\pm$ 0.31 \\
                            & NIFTY    & 95.31 $\pm$ 0.27 & 95.27 $\pm$ 0.19 & 95.81 $\pm$ 0.22 & 91.92 $\pm$ 0.24 \\ \hline
\end{tabular}

\caption{Distillation accuracy of different M2D variants for fair teacher}
\label{tab:fair_dist_acc}
\end{table}

\begin{table}[htbp]
\resizebox{\textwidth}{!}{
\begin{tabular}{l|l|llllll}
                          & Teacher & GCN(Feat.) & GCN(Adj.)  & GCN(Both)  & GCN(w/o)   & MLP(Feat.) & MLP(w/o)   \\ \hline
\textsc{Cora}                      & GAT     & 95.94 $\pm$ 0.12 & 91.05 $\pm$ 0.09 & 96.12 $\pm$ 0.14 & 90.56 $\pm$ 0.10 & 68.12 $\pm$ 0.13 & 63.58 $\pm$ 0.22 \\
                          & GT      & 93.08 $\pm$ 0.14 & 91.15 $\pm$ 0.12 & 93.14 $\pm$ 0.12 & 89.15 $\pm$ 0.12 & 79.12 $\pm$ 0.08 & 75.24 $\pm$ 0.13 \\ \hline
\multirow{2}{*}{\textsc{Citeseer}} & GAT     & 94.14 $\pm$ 0.15 & 95.12 $\pm$ 0.11 & 97.56 $\pm$ 0.12 & 92.41 $\pm$ 0.13 & 80.16 $\pm$ 0.15 & 78.23 $\pm$ 0.14 \\
                          & GT      & 94.12 $\pm$ 0.13 & 95.31 $\pm$ 0.12 & 93.12 $\pm$ 0.15 & 91.57 $\pm$ 0.07 & 83.12 $\pm$ 0.14 & 82.85 $\pm$ 0.18 \\ \hline
\multirow{2}{*}{\textsc{Photo}}    & GAT     & 96.13 $\pm$ 0.05 & 94.14 $\pm$ 0.18 & 93.89 $\pm$ 0.11 & 90.91 $\pm$ 0.21 & 76.12 $\pm$ 0.20 & 73.15 $\pm$ 0.31 \\
                          & GT      & 95.18 $\pm$ 0.14 & 95.54 $\pm$ 0.15 & 93.12 $\pm$ 0.20 & 89.51 $\pm$ 0.11 & 84.12 $\pm$ 0.13 & 80.12 $\pm$ 0.23 \\ \hline
\multirow{2}{*}{\textsc{Cornell}}  & GAT     & 95.12 $\pm$ 0.13 & 95.04 $\pm$ 0.12 & 92.56 $\pm$ 0.12 & 90.80 $\pm$ 0.13 & 74.13 $\pm$ 0.21 & 71.02 $\pm$ 0.28 \\
                          & GT      & 80.41 $\pm$ 0.11 & 75.13 $\pm$ 0.15 & 73.14 $\pm$ 0.17 & 73.01 $\pm$ 0.12 & 86.14 $\pm$ 0.12 & 84.12 $\pm$ 0.21 \\ \hline
\multirow{2}{*}{\textsc{Texas}}    & GAT     & 97.12 $\pm$ 0.16 & 95.35 $\pm$ 0.13 & 96.10 $\pm$ 0.14 & 94.30 $\pm$ 0.19 & 79.12 $\pm$ 0.13 & 77.21 $\pm$ 0.25 \\
                          & GT      & 86.13 $\pm$ 0.14 & 79.12 $\pm$ 0.14 & 74.21 $\pm$ 0.13 & 69.13 $\pm$ 0.14 & 90.63 $\pm$ 0.21 & 87.12 $\pm$ 0.19 \\ \hline
\end{tabular}}
\caption{Distillation accuracy of different M2D variants for the GAT and GT teacher}
\label{tab:dist_acc}
\end{table}

\section*{Additional Results and Visualization}
Table \ref{tab:gat_gt_acc_mlp} compares node classification accuracy of teacher models (GAT, GT) and student MLP under different distillation settings. On homophilic datasets such as \textsc{Cora}, \textsc{Citeseer}, and \textsc{Photo}, M2D consistently improves student performance over both vanilla training and standard distillation, demonstrating effective knowledge transfer from the teacher. However, on heterophilic datasets such as \textsc{Cornell} and \textsc{Texas}, the vanilla MLP already outperforms the teacher models. 
In these cases, M2D leads to performance degradation, as the student is encouraged to mimic suboptimal teacher representations. 
This highlights that the effectiveness of distillation depends on the quality of the teacher and that blindly transferring knowledge can be detrimental when the teacher is weaker than the student. 

Table~\ref{tab:fair_dist_acc} reports the distillation accuracy of different M2D variants with fairness-aware GNN teachers. Compared with the vanilla GCN without augmentation, M2D consistently improves the student's ability to approximate the teacher. Joint feature and adjacency augmentation achieves the best or near-best performance in most settings, while feature augmentation alone remains competitive and frequently outperforms adjacency-only augmentation. These results demonstrate that M2D effectively transfers the behavior of fairness-aware teachers to simpler student architectures.

Table~\ref{tab:dist_acc} evaluates distillation across different M2D variants, teacher architectures, and datasets. On the homophilic datasets (\textsc{Cora}, \textsc{Citeseer}, and \textsc{Photo}), GCN students generally achieve high distillation accuracy when learning from both GAT and GT teachers, whereas MLP students perform substantially worse. In contrast, on the heterophilic datasets (\textsc{Cornell} and \textsc{Texas}), MLP students achieve higher distillation accuracy in several settings, particularly when distilling from GT, while GCN performance degrades in some cases. These results indicate that the effectiveness of distillation depends on the compatibility between the student architecture and the structural properties of the underlying graph.
\begin{figure}[htbp]
    \centering
    \includegraphics[width=\textwidth]{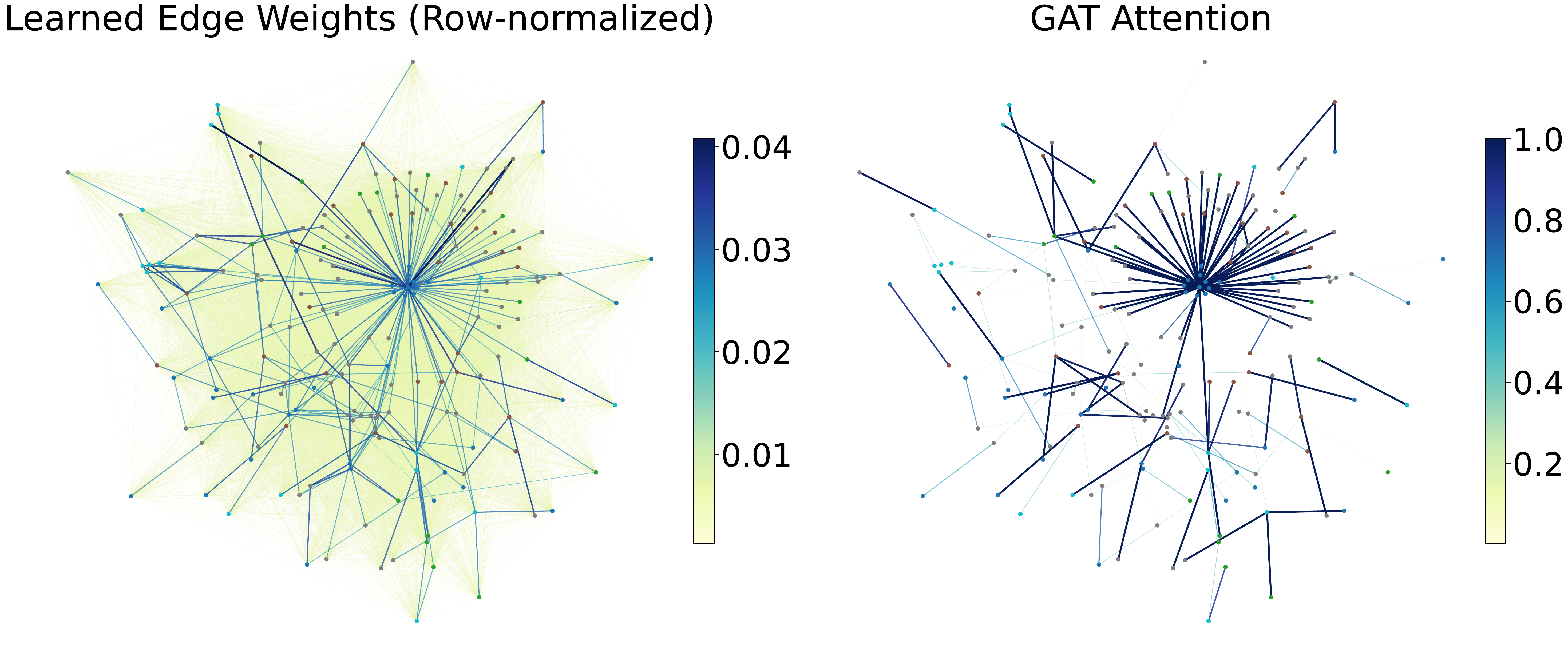}
    \caption{Visualization of edge importance on \textsc{Cornell} (heterophilic) using row-normalized M2D edge weights (left) and layer-1 GAT attention coefficients (right). Edge color and width indicate magnitude, with darker and thicker edges denoting higher importance. Learned weights are row-normalized (sum to one), enabling direct probabilistic comparison with attention. Both panels share an identical layout. The left panel highlights how M2D reweights edges to approximate GAT behavior. However, the Pearson correlation coefficient between the edge weights is $0.11$ (low).}
    \label{fig:adjCornell}
\end{figure}

Figure \ref{fig:adjCornell} compares edge importance from M2D and GAT on the heterophilic \textsc{Cornell} graph. While GAT directly learns attention coefficients, M2D learns edge weights that we row-normalize for comparison, enabling a probabilistic interpretation. The visualization reveals clear differences in how edges are weighted across the two methods. Due to heterophily, the normalized M2D weights do not match GAT attention exactly but instead reflect the adjustments required for the student to approximate the teacher’s behavior. This highlights the role of edge reweighting in aligning the student with attention-based representations. The difference in scale between the two plots is due to the hyperparameter $\gamma$. For the \textsc{Cornell} dataset, a higher value of the hyperparameter (0.45) generates the best result.

We further analyze the effect of $\alpha$ and $\lambda_{\mathrm{sparse}}$ on the magnitude of changes introduced by M2D. Figures~\ref{fig:hyper_alpha} and~\ref{fig:hyper_lambda} report the KL divergence between student and teacher predictions together with the average Frobenius distance between the original and distilled features and adjacency matrices, respectively. Across datasets, M2D achieves substantial reductions in KL divergence with only small average changes to both features and structure. As the regularization is relaxed, larger modifications initially improve distillation, but the benefit saturates, with further changes yielding little additional reduction in KL divergence. These results indicate that M2D can encode teacher behavior through relatively small modifications to the original graph and does not require increasingly large data transformations.

\begin{figure}
    \centering
    \includegraphics[width=\textwidth]{files/images/hyper_alpha1.pdf}
    \caption{Effect of the hyperparameter $\alpha$ on distillation across \textsc{Cora}, \textsc{Citeseer}, \textsc{Cornell}, \textsc{Texas}, and \textsc{Photo}. The KL divergence between student and teacher predictions and the average Frobenius norm between the original and distilled node features are reported for different values of $\alpha$. The results show that, beyond a certain level of feature change, further changes do not yield additional improvements in KL divergence, indicating diminishing returns from increasing $\alpha$.}
    \label{fig:hyper_alpha}
\end{figure}

\begin{figure}[htbp]
    \centering
    \includegraphics[width=\textwidth]{files/images/hyper_graph_lambda.pdf}
    \caption{Effect of the hyperparameter $\lambda_{\mathrm{sparse}}$ on distillation across \textsc{Cora}, \textsc{Citeseer}, \textsc{Cornell}, \textsc{Texas}, and \textsc{Photo}. The KL divergence between student and teacher predictions and the average Frobenius norm between the original and distilled adjacency matrices are reported for different values of $\lambda_{\mathrm{sparse}}$. The results show that, beyond a certain level of adjacency change, further changes do not yield additional improvements in KL divergence, indicating diminishing returns from increasing $\lambda_{\mathrm{sparse}}$.}
    \label{fig:hyper_lambda}
\end{figure}

\begin{table}[htbp]
\resizebox{\textwidth}{!}{
\begin{tabular}{ll|llll|llll}
                                                 &          & \multicolumn{4}{l|}{GCN}                         & \multicolumn{4}{l}{SGC}                         \\ \cline{3-10} 
                                                 &          & Accuracy $\uparrow$   & Fidelity $\uparrow$  & DP $\downarrow$        & EQOP $\downarrow$      & Accuracy $\uparrow$  & Fidelity $\uparrow$  & DP $\downarrow$         & EQOP $\downarrow$     \\ \cline{2-10} 
\multicolumn{1}{l|}{\multirow{4}{*}{\textsc{Simulation}}} 
& Vanilla  
& $96.58 \pm 0.14$ & $97.08 \pm 0.21$ & $42.42 \pm 0.41$ & $4.21 \pm 0.37$ 
& $95.41 \pm 0.25$ & $95.83 \pm 0.19$ & $40.21 \pm 0.29$ & $5.83 \pm 0.38$ \\

\multicolumn{1}{l|}{}                            
& $\tilde{\mathcal{G}}$(Feat.) 
& $95.47 \pm 0.32$ & $97.91 \pm 0.19$ & $32.07 \pm 0.34$ & $2.59 \pm 0.30$ 
& $95.23 \pm 0.51$ & $96.09 \pm 0.31$ & $35.47 \pm 0.41$ & $\mathbf{3.91 \pm 0.15}$ \\

\multicolumn{1}{l|}{}                            
& $\tilde{\mathcal{G}}$(Adj)   
& $95.06 \pm 0.28$ & $\mathbf{98.04 \pm 0.31}$ & $\mathbf{30.29 \pm 0.27}$ & $\mathbf{2.47 \pm 0.15}$ 
& $94.98 \pm 0.39$ & $\mathbf{96.72 \pm 0.29}$ & $\mathbf{31.27 \pm 0.36}$ & $4.02 \pm 0.27$ \\

\multicolumn{1}{l|}{}                            
& $\tilde{\mathcal{G}}$(both)  
& $\mathbf{96.00 \pm 0.22}$ & $98.02 \pm 0.25$ & $30.70 \pm 0.41$ & $2.61 \pm 0.37$ 
& $\mathbf{95.47 \pm 0.41}$ & $96.53 \pm 0.29$ & $33.05 \pm 0.21$ & $3.99 \pm 0.31$ \\ \hline

\multicolumn{1}{l|}{\multirow{4}{*}{\textsc{German}}}     
& Vanilla  
& $74.06 \pm 0.51$ & $93.21 \pm 0.29$ & $9.91 \pm 0.53$ & $7.18 \pm 0.40$ 
& $74.01 \pm 0.15$ & $92.79 \pm 0.23$ & $8.57 \pm 0.25$ & $8.01 \pm 0.26$ \\

\multicolumn{1}{l|}{}                            
& $\tilde{\mathcal{G}}$(Feat)  
& $72.31 \pm 0.19$ & $94.57 \pm 0.31$ & $4.57 \pm 0.26$ & $2.73 \pm 0.41$ 
& $\mathbf{74.23 \pm 0.32}$ & $93.58 \pm 0.39$ & $5.02 \pm 0.43$ & $5.00 \pm 0.31$ \\

\multicolumn{1}{l|}{}                            
& $\tilde{\mathcal{G}}$(Adj)   
& $\mathbf{72.45 \pm 0.26}$ & $94.98 \pm 0.14$ & $\mathbf{3.98 \pm 0.37}$ & $2.93 \pm 0.18$ 
& $73.97 \pm 0.28$ & $93.09 \pm 0.24$ & $4.98 \pm 0.24$ & $5.71 \pm 0.11$ \\

\multicolumn{1}{l|}{}                            
& $\tilde{\mathcal{G}}$(Both)  
& $72.38 \pm 0.22$ & $\mathbf{95.02 \pm 0.37}$ & $4.68 \pm 0.23$ & $\mathbf{2.54 \pm 0.30}$ 
& $74.09 \pm 0.18$ & $\mathbf{94.49 \pm 0.35}$ & $\mathbf{4.01 \pm 0.34}$ & $\mathbf{4.82 \pm 0.18}$ \\

\hline
\end{tabular}}
\caption{Generalization of M2D with a FairGNN teacher on \textsc{Simulation} and \textsc{German} using GCN and SGC as downstream architectures trained only on the augmented graph. Fidelity consistently increases with distillation, indicating improved alignment between student and teacher predictions, while fairness is generally improved, and predictive performance remains comparable.}
\label{tab:generalization_fairness_rest}
\end{table}

Tables \ref{tab:generalization_fairness_rest} and \ref{tab:generalization_gat_rest} evaluate whether the augmented graphs learned by M2D generalize beyond the student architecture used during distillation. We train new downstream models solely on the learned augmented graphs, without access to the teacher during downstream training. With the FairGNN teacher on \textsc{Simulation} and \textsc{German}, both GCN and SGC consistently achieve higher fidelity than their vanilla counterparts, while generally reducing demographic parity and equal opportunity gaps and maintaining comparable predictive performance. Similarly, with the GAT teacher on \textsc{Citeseer}, \textsc{Photo}, \textsc{Cornell}, and \textsc{Texas}, feature adjacency and joint augmentation consistently improve the fidelity of downstream GCN and SGC models. These results suggest that M2D encodes teacher behavior into the augmented graph itself rather than tying the transferred knowledge to a particular student architecture, allowing the learned graph to serve as a reusable representation for different downstream models.

\begin{table}[htbp]
\resizebox{\textwidth}{!}{
\begin{tabular}{ll|ll|ll|ll}
                                                &          & \multicolumn{2}{l|}{GCN} & \multicolumn{2}{l|}{SGC} & \multicolumn{2}{l}{MLP} \\ \cline{3-8} 
                                                &          & Accuracy    & Fidelity   & Accuracy    & Fidelity   & Accuracy & Fidelity \\ \hline
\multicolumn{1}{l|}{}                           & Vanilla  & $71.64 \pm 0.16$ & $91.58 \pm 0.19$ & $70.01 \pm 0.38$ & $91.75 \pm 0.30$ & $68.71 \pm 0.14$ & $77.25 \pm 0.36$ \\
\multicolumn{1}{l|}{}                           & $\tilde{\mathcal{G}}$(Feat.) & $\mathbf{71.98 \pm 0.35}$ & $93.69 \pm 0.33$ & $70.59 \pm 0.33$ & $\mathbf{92.91 \pm 0.31}$ & $69.29 \pm 0.25$ & $79.91 \pm 0.24$ \\
\multicolumn{1}{l|}{}                           & $\tilde{\mathcal{G}}$(Adj)   & $71.70 \pm 0.29$ & $93.50 \pm 0.19$ & $70.47 \pm 0.22$ & $92.41 \pm 0.23$ &  &  \\
\multicolumn{1}{l|}{\multirow{-4}{*}{\textsc{Citeseer}}} & $\tilde{\mathcal{G}}$(both)  & $71.87 \pm 0.24$ & $\mathbf{93.91 \pm 0.25}$ & $\mathbf{70.91 \pm 0.27}$ & $92.49 \pm 0.29$ &  &  \\ \hline

\multicolumn{1}{l|}{}                           & Vanilla  & $90.23 \pm 0.23$ & $93.41 \pm 0.17$ & $89.23 \pm 0.18$ & $92.39 \pm 0.23$ & $85.03 \pm 0.09$ & $71.41 \pm 0.39$ \\
\multicolumn{1}{l|}{}                           & $\tilde{\mathcal{G}}$(Feat)  & $91.19 \pm 0.15$ & $95.47 \pm 0.14$ & $90.18 \pm 0.31$ & $94.71 \pm 0.35$ & $86.91 \pm 0.23$ & $73.14 \pm 0.24$ \\
\multicolumn{1}{l|}{}                           & $\tilde{\mathcal{G}}$(Adj)   & $91.03 \pm 0.25$ & $95.81 \pm 0.41$ & $90.03 \pm 0.18$ & $94.59 \pm 0.37$ &  &  \\
\multicolumn{1}{l|}{\multirow{-4}{*}{\textsc{Photo}}}    & $\tilde{\mathcal{G}}$(Both)  & $\mathbf{91.57 \pm 0.21}$ & $\mathbf{96.03 \pm 0.23}$ & $\mathbf{91.10 \pm 0.31}$ & $\mathbf{94.79 \pm 0.23}$ &  &  \\ \hline

\multicolumn{1}{l|}{}                           & Vanilla  & $53.61 \pm 0.41$ & $94.29 \pm 0.35$ & $51.52 \pm 0.11$ & $89.57 \pm 0.23$ & $66.34 \pm 0.18$ & $68.24 \pm 0.32$ \\
\multicolumn{1}{l|}{}                           & $\tilde{\mathcal{G}}$(Feat)  & $54.19 \pm 0.28$ & $95.87 \pm 0.17$ & $52.64 \pm 0.27$ & $\mathbf{93.72 \pm 0.34}$ & $63.14 \pm 0.29$ & $70.17 \pm 0.22$ \\
\multicolumn{1}{l|}{}                           & $\tilde{\mathcal{G}}$(Adj)   & $54.20 \pm 0.32$ & $\mathbf{96.91 \pm 0.20}$ & $\mathbf{52.97 \pm 0.28}$ & $92.21 \pm 0.26$ &  &  \\
\multicolumn{1}{l|}{\multirow{-4}{*}{\textsc{Cornell}}}  & $\tilde{\mathcal{G}}$(Both)  & $\mathbf{54.79 \pm 0.25}$ & $96.12 \pm 0.43$ & $52.39 \pm 0.25$ & $92.57 \pm 0.26$ & &  \\ \hline

\multicolumn{1}{l|}{}                           & Vanilla  & $62.36 \pm 0.20$ & $95.43 \pm 0.20$ & $60.29 \pm 0.32$ & $91.91 \pm 0.32$ & $75.27 \pm 0.20$ & $73.18 \pm 0.24$ \\
\multicolumn{1}{l|}{}                           & $\tilde{\mathcal{G}}$(Feat)  & $63.21 \pm 0.17$ & $95.86 \pm 0.36$ & $61.98 \pm 0.21$ & $94.49 \pm 0.34$ & $72.29 \pm 0.17$ & $75.57 \pm 0.18$ \\
\multicolumn{1}{l|}{}                           & $\tilde{\mathcal{G}}$(Adj)   & $\mathbf{65.41 \pm 0.15}$ & $\mathbf{96.73 \pm 0.27}$ & $62.30 \pm 0.37$ & $\mathbf{95.03 \pm 0.30}$ &  &  \\
\multicolumn{1}{l|}{\multirow{-4}{*}{\textsc{Texas}}}    & $\tilde{\mathcal{G}}$(Both)  & $65.15 \pm 0.26$ & $96.08 \pm 0.36$ & $\mathbf{63.59 \pm 0.13}$ & $94.93 \pm 0.19 $ &  &  \\ 
\hline
\end{tabular}}
\caption{Generalization of M2D with a GAT teacher on \textsc{Citeseer}, \textsc{Photo}, \textsc{Cornell} and \textsc{Texas} using different downstream architectures trained only on the augmented graph. We evaluate M2D with feature and edge augmentation individually and jointly across GCN, SGC, and MLP. Fidelity consistently increases with distillation, indicating improved alignment between the downstream models and teacher predictions.}
\label{tab:generalization_gat_rest}
\end{table}

Table \ref{tab:modularFair_mlp} evaluates whether the fairness behavior of a FairGNN teacher can be transferred to an MLP through M2D feature augmentation. Across \textsc{Simulation}, \textsc{German}, and NBA, training the MLP on the augmented features consistently improves fidelity to the teacher while reducing both demographic parity and equal opportunity gaps. On \textsc{Simulation} and \textsc{German}, these fairness gains are achieved with only a modest reduction in accuracy, while on NBA, accuracy also improves. These results show that M2D can encode fairness-related behavior directly into the transformed features, enabling a simple downstream MLP to inherit the teacher's fairer prediction behavior without relying on graph structure.

\begin{table}[htbp]
\centering
\resizebox{0.85\columnwidth}{!}{
\begin{tabular}{ll|llll}
         &          & \multicolumn{4}{c}{MLP} \\ \cline{3-6}
         &          & Accuracy ($\uparrow$) & Fidelity ($\uparrow$) & DP ($\downarrow$) & EQOP ($\downarrow$) \\ \hline

\multirow{2}{*}{\textsc{Simulation}}
& Vanilla
& $96.23 \pm 0.17$ & $95.91 \pm 0.22$ & $42.98 \pm 0.33$ & $6.02 \pm 0.29$ \\

& $\tilde{\mathcal G}$(Feat.)
& $95.02 \pm 0.40$ & $97.00 \pm 0.25$ & $34.95 \pm 0.30$ & $5.21 \pm 0.34$ \\ \hline

\multirow{2}{*}{\textsc{German}}
& Vanilla
& $76.58 \pm 0.22$ & $91.03 \pm 0.44$ & $7.39 \pm 0.19$ & $8.21 \pm 0.29$ \\

& $\tilde{\mathcal G}$(Feat.)
& $74.59 \pm 0.41$ & $93.47 \pm 0.14$ & $5.41 \pm 0.25$ & $4.49 \pm 0.33$ \\ \hline

\multirow{2}{*}{NBA}
& Vanilla
& $62.51 \pm 0.42$ & $75.23 \pm 0.42$ & $6.81 \pm 0.61$ & $14.13 \pm 0.46$ \\

& $\tilde{\mathcal G}$(Feat.)
& $64.25 \pm 0.91$ & $77.02 \pm 0.81$ & $4.51 \pm 0.81$ & $12.12 \pm 0.51$ \\

\end{tabular}}
\caption{Generalization of M2D using an MLP student across fairness datasets. We evaluate M2D with feature augmentation. Feature distillation generally improves fidelity and reduces fairness disparities, demonstrating improved alignment with the teacher while maintaining comparable predictive performance.}
\label{tab:modularFair_mlp}
\end{table}

Figure \ref{fig:fair_feat_remaining} further examines the KDE distributions of the 20 remaining (see Figure \ref{fig:fair_feat}) original ($\mathbf{X}$) and learned ($\tilde{\mathbf{X}}$) features across the sensitive groups on the \textsc{German} dataset. The first and third rows show the distributions of the original features, while the second and fourth rows show their corresponding learned versions after feature masking. Across the features, the learned representations exhibit varying degrees of change, with several features becoming concentrated near zero. This concentration indicates that the learned masks suppress the corresponding feature contributions in the student representation. At the same time, features whose distributions remain similar before and after masking are largely preserved, suggesting that M2D selectively modifies feature contributions rather than uniformly altering the input representation. Together, these results provide further evidence that the learned feature masks adaptively modify the node features to transfer the teacher's behavior to the student.

\begin{figure}[htbp]
    \centering
    \includegraphics[width=\textwidth]{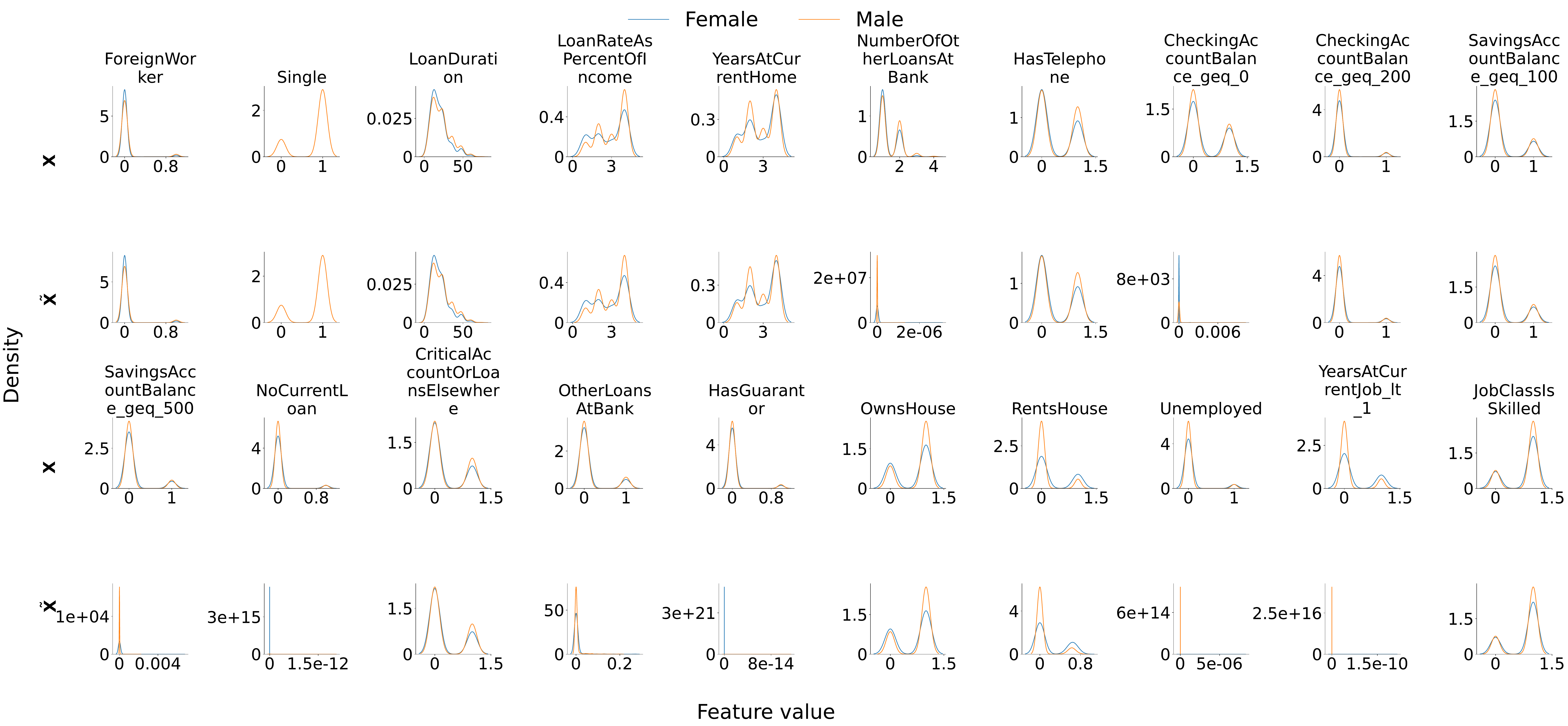}
    \caption{KDE distributions of the 20 remaining (see Figure \ref{fig:fair_feat}) original ($\mathbf{X}$) and learned ($\tilde{\mathbf{X}}$) features across sensitive groups on the \textsc{German} dataset. The first and third rows show the original features, while the second and fourth rows show their corresponding learned versions. Features in $\tilde{\mathbf{X}}$ concentrated around zero correspond to features whose contributions are suppressed by the learned masks, thereby reducing their impact on the representation. Teacher: FairGNN, Student: GCN (Feat.).}
    \label{fig:fair_feat_remaining}
\end{figure}

\begin{figure}[htbp]
    \centering
    \includegraphics[width=\textwidth]{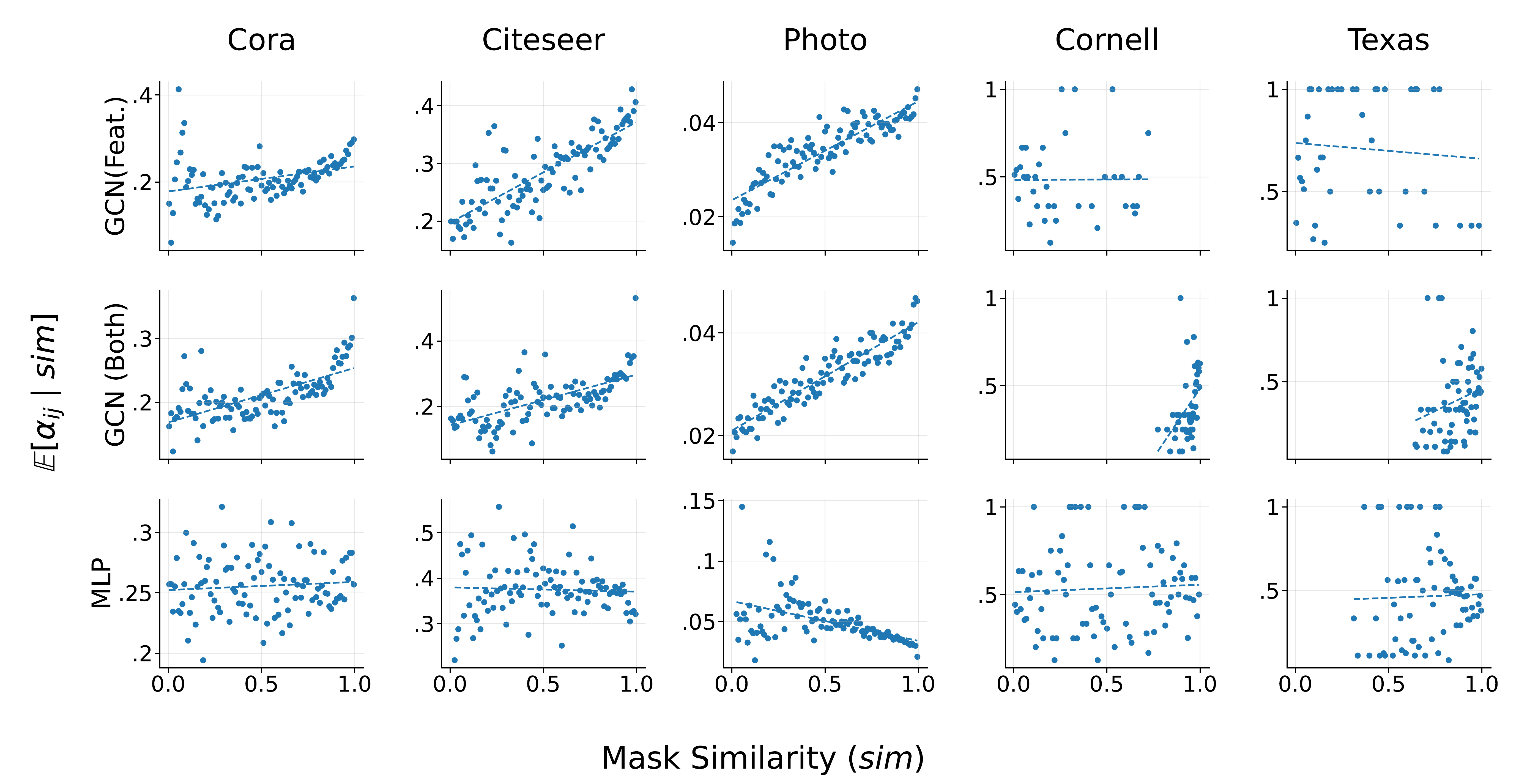}
    \caption{Analysis of learned feature mask with respect to GT attention. We bin mask similarity $\mathrm{sim}(\mathbf{m}_i, \mathbf{m}_j)$ and compute average attention per bin. GCN students exhibit a clear increasing relationship in homophilic datasets (\textsc{Cora}, \textsc{Citeseer}, \textsc{Photo}), indicating alignment with teacher attention. This trend weakens in heterophilic datasets due to lower distillation accuracy. For MLP students, the trend appears only when distillation accuracy is high.}
    \label{fig:feat_attn_gt}
\end{figure}

In Figure \ref{fig:feat_attn_gt}, we analyze how the learned feature mask aligns with the attention mechanism of the GT teacher by examining the relationship between feature similarity and attention weights. Specifically, we partition node pairs into bins based on similarity $\mathrm{sim}(\mathbf{z}_i, \mathbf{z}_j) $ and compute the average attention within each bin (Similar to Figure \ref{fig:gat_feat}). In homophilic datasets (\textsc{Cora}, \textsc{Citeseer}, \textsc{Photo}), GCN students exhibit a clear positive relationship, indicating that higher attention weights correspond to more similar learned features. In contrast, this relationship weakens in heterophilic datasets (\textsc{Cornell}, \textsc{Texas}), due to lower distillation accuracy (See Table \ref{tab:dist_acc}). For MLP students, the trend is absent in homophilic datasets due to poor distillation performance but emerges in heterophilic settings where the student better matches the teacher. These results indicate that higher feature similarity is associated with higher attention weights when distillation accuracy is sufficiently high.

\begin{figure}[htbp]
    \centering
    \includegraphics[width=0.9\textwidth]{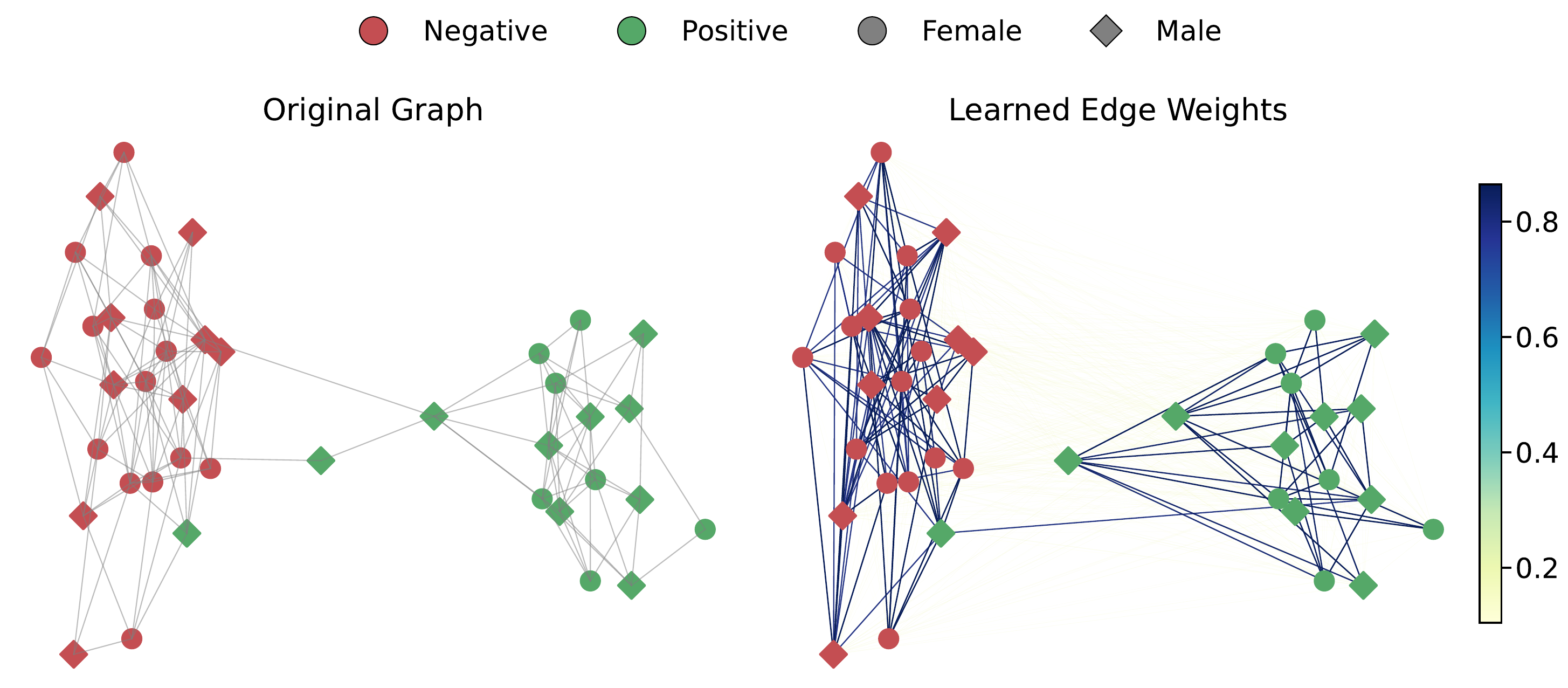}
    \caption{Original and augmented structure for a sample of 40 nodes from the \textsc{German} dataset, highlighting edges that are modified for fair prediction. Edge weights are encoded by color and width. Node colors denote class labels, and shapes represent sensitive attributes.}
    \label{fig:fair_adj}
\end{figure}

Figure \ref{fig:fair_adj} shows how structure learning augments a 40-node subgraph from the same dataset. Compared to the original graph (left), the learned graph (right) shows removed or weakened biased connections and strengthened or added unbiased ones. For example, an edge between a male positive node and a male negative node is removed to prevent the male negative node from being biased toward a positive prediction. 

\section*{Complexity Analysis}
The feature learner takes the node embeddings $\mathbf{H}\in\mathbb{R}^{n\times d_h}$ produced by the GNN and maps them to a feature mask $\mathbf{M}\in\mathbb{R}^{n\times d}$, where $d$ is the dimensionality of the original node features. For a two-layer MLP with hidden dimension $d_m$, generating the mask requires $\mathcal{O}(nd_hd_m+ nd_md)$ operations. The resulting mask is then applied element-wise to the original feature matrix $\mathbf{X}\in\mathbb{R}^{n\times d}$. In contrast, the structure learner also takes the node embeddings $\mathbf{H}\in\mathbb{R}^{n\times d_h}$ but generates an edge score for each pair of nodes, resulting in an $n\times n$ adjacency matrix. Consequently, its pairwise edge generation requires $O(n^2d_h)$ operations for a direct pairwise scoring function. Therefore, the complexity becomes $\mathcal{O}(nd_hd_m + n^2 d_m)$. Thus, feature learning scales linearly with the number of nodes, whereas structure learning scales quadratically with the number of nodes due to its dense pairwise edge generation. The quadratic complexity can be reduced by restricting adjacency learning to a selected subset of node pairs, rather than evaluating all $n^2$ possible edges, following prior work on scalable graph structure learning \cite{tang2015line, shin2023efficient, saha2023learning}.


\section*{Reproducibility}
Our code is available in an anonymous GitHub repository: https://anonymous.4open.science/r/m2d-1F9A